\documentclass[draftcls,onecolumn]{IEEEtran}
\usepackage{xr}
\usepackage{cite}
\usepackage{amsmath}
\usepackage[colorlinks=true,linkcolor=black,citecolor=black,urlcolor=black]{hyperref}
\usepackage{cases}
\usepackage[utf8]{inputenc}
\usepackage[english]{babel}
\usepackage{footnote}
\usepackage{booktabs}

\usepackage{bm}

\usepackage[mathscr]{eucal}
\usepackage{epsfig,epsf,psfrag}
\usepackage{amssymb,amsmath,amsfonts,latexsym}
\usepackage{amsmath,graphicx,bm,xcolor,url}
\usepackage[caption=false]{subfig} 
\usepackage{fixltx2e}
\usepackage{array}
\usepackage{verbatim}
\usepackage{bm}
\usepackage{algpseudocode}
\usepackage{algorithm}
\usepackage{verbatim}
\usepackage{textcomp}
\usepackage{mathrsfs}
\usepackage{epstopdf}
\usepackage{relsize}
\usepackage{cleveref} 
\usepackage{subfig}
 \usepackage{amsthm}

\catcode`~=11 \def\UrlSpecials{\do\~{\kern -.15em\lower .7ex\hbox{~}\kern .04em}} \catcode`~=13 

\allowdisplaybreaks[3]

\newcommand{\calC}{\mathcal{C}}
\newcommand{\calD}{\mathcal{D}}

\newcommand{\calL}{\mathcal{L}}

\newcommand{\calN}{\mathcal{N}}

\newcommand{\calP}{\mathcal{P}}

\newcommand{\calS}{\mathcal{S}}
\newcommand{\calT}{\mathcal{T}}
\newcommand{\calU}{\mathcal{U}}
\newcommand{\calV}{\mathcal{V}}

\newcommand{\calX}{\mathcal{X}}

\newcommand{\ba}{\mathbf{a}}
\newcommand{\bA}{\mathbf{A}}
\newcommand{\bb}{\mathbf{b}}

\newcommand{\bI}{\mathbf{I}}

\newcommand{\br}{\mathbf{r}}

\newcommand{\bs}{\mathbf{s}}

\newcommand{\bu}{\mathbf{u}}

\newcommand{\bv}{\mathbf{v}}

\newcommand{\bW}{\mathbf{W}}
\newcommand{\bx}{\mathbf{x}}
\newcommand{\bX}{\mathbf{X}}

\newcommand{\bz}{\mathbf{z}}

\newcommand{\rmd}{\mathrm{d}}

\newcommand{\rmH}{\mathrm{H}}

\newcommand{\rmS}{\mathrm{S}}

\newcommand{\bbP}{\mathbb{P}}

\newcommand{\bbR}{\mathbb{R}}

\DeclareMathAlphabet{\mathbsf}{OT1}{cmss}{bx}{n}
\DeclareMathAlphabet{\mathssf}{OT1}{cmss}{m}{sl}

\DeclareSymbolFont{bsfletters}{OT1}{cmss}{bx}{n}  
\DeclareSymbolFont{ssfletters}{OT1}{cmss}{m}{n}
\DeclareMathSymbol{\bsfGamma}{0}{bsfletters}{'000}
\DeclareMathSymbol{\ssfGamma}{0}{ssfletters}{'000}
\DeclareMathSymbol{\bsfDelta}{0}{bsfletters}{'001}
\DeclareMathSymbol{\ssfDelta}{0}{ssfletters}{'001}
\DeclareMathSymbol{\bsfTheta}{0}{bsfletters}{'002}
\DeclareMathSymbol{\ssfTheta}{0}{ssfletters}{'002}
\DeclareMathSymbol{\bsfLambda}{0}{bsfletters}{'003}
\DeclareMathSymbol{\ssfLambda}{0}{ssfletters}{'003}
\DeclareMathSymbol{\bsfXi}{0}{bsfletters}{'004}
\DeclareMathSymbol{\ssfXi}{0}{ssfletters}{'004}
\DeclareMathSymbol{\bsfPi}{0}{bsfletters}{'005}
\DeclareMathSymbol{\ssfPi}{0}{ssfletters}{'005}
\DeclareMathSymbol{\bsfSigma}{0}{bsfletters}{'006}
\DeclareMathSymbol{\ssfSigma}{0}{ssfletters}{'006}
\DeclareMathSymbol{\bsfUpsilon}{0}{bsfletters}{'007}
\DeclareMathSymbol{\ssfUpsilon}{0}{ssfletters}{'007}
\DeclareMathSymbol{\bsfPhi}{0}{bsfletters}{'010}
\DeclareMathSymbol{\ssfPhi}{0}{ssfletters}{'010}
\DeclareMathSymbol{\bsfPsi}{0}{bsfletters}{'011}
\DeclareMathSymbol{\ssfPsi}{0}{ssfletters}{'011}
\DeclareMathSymbol{\bsfOmega}{0}{bsfletters}{'012}
\DeclareMathSymbol{\ssfOmega}{0}{ssfletters}{'012}

\newcommand{\btheta}{\bm{\theta}}

\newcommand{\bxi}{\bm{\xi}}

\newcommand{\bone}{\mathbf{1}}

\theoremstyle{plain}
\newtheorem{theorem}{Theorem} 
\newtheorem{lemma}{Lemma}

\newtheorem{corollary}{Corollary}
\newtheorem{definition}{Definition} 

\newtheorem{remark}{Remark}

\newcommand{\qednew}{\nobreak \ifvmode \relax \else
      \ifdim\lastskip<1.5em \hskip-\lastskip
      \hskip1.5em plus0em minus0.5em \fi \nobreak
      \vrule height0.75em width0.5em depth0.25em\fi}

\begin{document}
    
\title{Sample Complexity Bounds for 1-bit Compressive Sensing and Binary Stable Embeddings with Generative Priors}

\author{Zhaoqiang Liu, Selwyn Gomes, Avtansh Tiwari, Jonathan Scarlett

\thanks{
Z.~Liu and S.~Gomes are with the Department of Computer Science, National University of Singapore (email: \url{dcslizha@nus.edu.sg}, \url{selwyn@comp.nus.edu.sg}). 


A.~Tiwari is with the Department of Electrical Engineering,  Indian Institute of Technology (IIT) Kanpur (email: \url{avtansh@iitk.ac.in}). 

J.~Scarlett is with the Department of Computer Science and the Department of Mathematics, National University of Singapore (email: \url{scarlett@comp.nus.edu.sg}).

This work was supported by the Singapore National Research Foundation (NRF) under grant number R-252-000-A74-281.}}

\maketitle

\begin{abstract}
    The goal of standard 1-bit compressive sensing is to accurately recover an unknown sparse vector from binary-valued measurements, each indicating the sign of a linear function of the vector.  Motivated by recent advances in compressive sensing with generative models, where a generative modeling assumption replaces the usual sparsity assumption, we study the problem of 1-bit compressive sensing with generative models. We first consider noiseless 1-bit measurements, and provide sample complexity bounds for approximate recovery under i.i.d.~Gaussian measurements and a Lipschitz continuous generative prior, as well as a near-matching algorithm-independent lower bound. Moreover, we demonstrate that the Binary $\epsilon$-Stable Embedding property, which characterizes the robustness of the reconstruction to measurement errors and noise, also holds for 1-bit compressive sensing with Lipschitz continuous generative models with sufficiently many Gaussian measurements. In addition, we apply our results to neural network generative models, and provide a proof-of-concept numerical experiment demonstrating significant improvements over sparsity-based approaches.
\end{abstract}

\section{Introduction}\label{sec:intro}

The compressive sensing (CS) problem~\cite{Fou13,wainwright2019high}, which aims to recover a {\em sparse} signal from a small number of linear measurements, is fundamental in machine learning, signal processing and statistics. It has been popular over the past 1--2 decades and has become increasingly well-understood, with theoretical guarantees including sharp performance bounds for both practical algorithms \cite{Wai09a,Don13,Ame14,Wen2016} and potentially intractable information-theoretically optimal algorithms \cite{Wai09,Ari13,Can13,Sca15}. 

Unlike conventional compressive sensing, which assumes infinite-precision real-valued measurements, in {\em 1-bit} compressive sensing~\cite{boufounos20081}, each measurement is quantized to a single bit, namely its sign. Considerable research effort has been placed to 1-bit compressive sensing~\cite{gupta2010sample,zhu2015towards,zhang2014efficient,gopi2013one,ai2014one,awasthi2016learning}, one motivation being that 1-bit quantization can be implemented in hardware with low cost and is robust to certain nonlinear distortions~\cite{boufounos2010reconstruction}. 

In addition, motivated by recent advances in deep generative models~\cite{Fos19}, a new perspective has recently emerged in CS, in which the sparsity assumption is replaced by the assumption that the underlying signal lies near the range of a suitably-chosen generative model, typically corresponding to a deep neural network~\cite{Bor17}. Along with several theoretical developments, it has been numerically verified that generative priors can reduce the number of measurements required for a given accuracy by large factors such as $5$ to $10$ \cite{Bor17}. 

In this paper, following the developments in both 1-bit CS and CS with generative priors, we establish a variety of fundamental theoretical guarantees for 1-bit compressive sensing using generative models. 

\vspace*{-1ex}
\subsection{Related Work}

{\bf Sparsity-based 1-bit compressive sensing:} The framework of 1-bit compressive sensing (CS) was introduced and studied in~\cite{boufounos20081}. Subsequently, various numerical algorithms were designed \cite{boufounos20081, boufounos2009greedy, boufounos2010reconstruction, laska2011trust, zymnis2009compressed}, often with convergence guarantees. In addition, several works have developed theoretical guarantees for support recovery and approximate vector recovery in 1-bit CS~\cite{gopi2013one,zhang2014efficient,zhu2015towards,acharya2017improved, plan2012robust, plan2013one}. In these works, it is usually assumed that the measurement matrix contains i.i.d.~Gaussian entries. Such an assumption is generalized to allow sub-Gaussian~\cite{ai2014one, dirksen2018non} and log-concave~\cite{awasthi2016learning} measurement matrices. A survey on 1-bit CS can be found in~\cite{li2018survey}.

To address the fact that standard 1-bit measurements give no information about the norm of the underlying signal vector, the so-called dithering technique, which adds artificial random noise before quantization, has been considered~\cite{knudson2016one,xu2018quantized,jacques2017time,dirksen2018non} to also enable the estimation of the norm. 

Perhaps most relevant to the present paper, \cite{jacques2013robust} studies the robustness of 1-bit CS by considering binary stable embeddings of sparse vectors. We seek to provide analogous theoretical guarantees to those in~\cite{jacques2013robust}, but with a generative prior in place of the sparsity assumption.  We adopt similar high-level proof steps, but with significantly different details.

{\bf Compressive sensing with generative models:} Bora {\em et al.}~\cite{Bor17} show that roughly $O(k \log L)$ random Gaussian linear measurements suffice for accurate recovery when the generative model is an $L$-Lipschitz function with bounded $k$-dimensional inputs. 
The analysis in~\cite{Bor17} is based on minimizing an empirical loss function. In practice, such a task may be hard, and the authors propose to use a simple gradient descent algorithm in the latent space. 
The theoretical analysis is based on showing that Gaussian random matrices satisfy a natural counterpart to the Restricted Eigenvalue Condition (REC) termed the Set-REC. 
Follow-up works of~\cite{Bor17} provide various additional algorithmic guarantees for compressive sensing with generative models~\cite{Sha18,peng2020solving,Dha18,Han18}, as well as information-theoretic lower bounds \cite{kamath2019lower,liu2020information}.

In a recent work, the authors of~\cite{qiu2019robust} study robust 1-bit compressive sensing with ReLU-based generative models. In particular, the authors design an empirical risk minimization algorithm, and prove that it is able to faithfully recover bounded target vectors produced by the model from quantized noisy measurements. Our results and those of \cite{qiu2019robust} are complementary to each other, with several differences in the setup:
\begin{itemize}
    \item In \cite{qiu2019robust}, the dithering technique is used, adding artificial random noise before quantization to enable the recovery the norm of the signal vector, whereas we do not consider the use of dithering. Both settings are of interest depending on whether dithering is feasible to implement in the application at hand.
    \item In \cite{qiu2019robust}, the focus is on ReLU networks without offset terms, whereas we consider general $L$-Lipschitz generative models.
    \item The pre-quantization noise in \cite{qiu2019robust} is assumed to be sub-exponential, whereas we allow for general and possibly adversarial noise.
    \item The theoretical analysis in~\cite{qiu2019robust} focuses on a particular recovery algorithm, whereas our results are information-theoretic in nature.
\end{itemize}

\vspace*{-1ex}
\subsection{Contributions}

In this paper, we establish a variety of fundamental theoretical guarantees for 1-bit compressive sensing using generative models.  Our main results are outlined as follows:
\begin{itemize}
    \item In Section~\ref{sec:ub_noiseless}, for noiseless measurements, we characterize the number of i.i.d.~Gaussian measurements sufficient (i.e., an upper bound on the sample complexity) to attain approximate recovery of the underlying signal under a Lipschitz continuous generative prior.
    \item In Section~\ref{sec:lb_noiseless}, for noiseless measurements, we show that our upper bound is nearly tight by giving a near-matching algorithm-independent lower bound for a particular Lipschitz continuous generative model.
    \item In Section~\ref{sec:ub_noisy}, we establish the Binary $\epsilon$-Stable Embedding (B$\epsilon$SE) property, which characterizes the reconstruction robustness to measurement errors and noise. Specifically, we characterize the number of i.i.d.~Gaussian measurements sufficient to ensure that this property holds. In Section~\ref{sec:nn_apply}, we specialize these results to feed-forward neural network generative models.
    \item In Section \ref{sec:algo}, we propose a practical iterative algorithm for 1-bit CS with generative priors, and demonstrate its effectiveness in a simple numerical example.
\end{itemize}

\vspace*{-1ex}
\subsection{Notation}\label{sec:notations}

We use upper and lower case boldface letters to denote matrices and vectors respectively. We write $[N]=\{1,2,\cdots,N\}$ for a positive integer $N$. A {\em generative model} is a function $G \,:\, \calD\to \bbR^n$, with latent dimension $k$, ambient dimension $n$, and input domain $\calD \subseteq \bbR^k$.
$\calS^{n-1} := \{\bx \in \bbR^n: \|\bx\|_2=1\}$ represents the unit sphere in $\bbR^n$. For $\bx,\bs \in \calS^{n-1}$, $\rmd_\rmS (\bx,\bs) := \frac{1}{\pi}\arccos \langle \bx, \bs \rangle$ denotes the geodesic distance, which is the normalized angle between vectors $\bx$ and $\bs$. For $\bv,\bv' \in \bbR^m$, $\rmd_{\rm H}(\bv,\bv') := \frac{1}{m}\sum_{i=1}^m \boldsymbol{1} \{ v_i \ne v'_i \}$ denotes the Hamming distance. 
We use $\|\bX\|_{2 \to 2}$ to denote the spectral norm of a matrix $\bX$. We define the $\ell_2$-ball $B_2^k(r):=\{\bz \in \bbR^k: \|\bz\|_2 \le r\}$, and the $\ell_{\infty}$-ball $B_\infty^k(r):=\{\bz \in \bbR^k: \|\bz\|_\infty \le r\}$.  For a set $B \subseteq \bbR^k$ and a generative model $G \,:\,\bbR^k \to \bbR^n$, we write $G(B) = \{ G(\bz) \,:\, \bz \in B  \}$.

\section{Noiseless Measurements} \label{sec:noiseless}

In this section, we derive near-matching upper and lower bounds on the sample complexity in the noiseless setting, in which the measurements take the form
\begin{equation}
    \bb = \Phi(\bx) := \mathrm{sign}(\bA\bx) \label{eq:measurements}
\end{equation}
for some measurement matrix $\bA \in \bbR^{m \times n}$ and unknown underlying signal $\bx \in G(B_2^k(r))$, where $G \,:\,B_2^k(r) \to \calS^{n-1}$ is the generative model. 

\begin{remark} 
    In the following, for clarity, we will assume that the range of the generative model is contained in the unit sphere, i.e., $G(B_2^k(r)) \subseteq \calS^{n-1}$, and provide guarantees of the form $\|\bx - \hat{\bx}\|_2 \le \epsilon$ for some estimate $\hat{\bx}$.  While the preceding assumption may appear restrictive, these results readily transfer to any general (unnormalized) generative model $\tilde{G}$ with $\tilde{G}(B_2^k(r)) \subseteq \bbR^n$ when the recovery guarantee is modified to $\big\| \frac{\bx}{\|\bx\|_2} - \frac{\hat{\bx}}{\|\hat{\bx}\|_2}\big\|_2 \le \epsilon$ (note that norm estimation is impossible under 1-bit measurements of the form \eqref{eq:measurements}).  The idea is to apply our results to $G(\bx) = \frac{\tilde{G}(\bx)}{\|\tilde{G}(\bx)\|_2}$; see Section~\ref{sec:nn_apply} for an example and further discussion.

    In addition, we consider spherical domains with radius $r$.  Similar to that in~\cite{Bor17}, the assumption of a bounded domain is mild, since the dependence on $r$ in the sample complexity will only be logarithmic.  In addition, our lower bound will show that such a dependence on $r$ is unavoidable.
\end{remark}

\subsection{Upper Bound}\label{sec:ub_noiseless}

Our first main result shows that with sufficiently many independent Gaussian measurements, with high probability, any two signals separated by some specified distance $\epsilon$ produce distinct measurements.  This amounts to an upper bound on the sample complexity for noiseless 1-bit recovery.

\begin{theorem}\label{thm:main_noiseless}
     Fix $r > 0$ and $\epsilon \in (0,1)$, and let $\bA \in \bbR^{m\times n}$ be generated as $A_{ij} \overset{i.i.d.}{\sim} \calN(0,1)$. Suppose that the generative model $G: \bbR^k \rightarrow \bbR^n$ is $L$-Lipschitz and $G(B_2^k(r)) \subseteq \calS^{n-1}$.  For $m = \Omega\left(\frac{k}{\epsilon}\log \frac{L r }{\epsilon^2}\right)$,\footnote{In all statements of the form $m = \Omega(\cdot)$ in our upper bounds, the implied constant is implicitly assumed to be sufficiently large.} with probability at least $1-e^{-\Omega(\epsilon m)}$, we have for all $\bx,\bs \in G(B_2^k(r))$ that
     \begin{equation}
          \|\bx-\bs\|_2 > \epsilon \Rightarrow \rmd_\rmH(\Phi(\bx),\Phi(\bs)) = \Omega(\epsilon).
     \end{equation}
    In particular, if $\|\bx - \bs\|_2 > \epsilon$, then $\Phi(x) \ne \Phi(\bs)$.
\end{theorem}

The proof is outlined below, with the full details given in the supplementary material.  From Theorem~\ref{thm:main_noiseless}, we immediately obtain the following corollary giving a recovery guarantee for noiseless 1-bit compressive sensing.

\begin{corollary}
     Let $\bA$ and $G$ follow the same assumptions as those given in Theorem~\ref{thm:main_noiseless}. Then, for a fixed $\epsilon \in (0,1)$, when $m = \Omega\left(\frac{k}{\epsilon}\log \frac{L r }{\epsilon^2}\right)$, the following holds with probability at least $1-e^{-\Omega(\epsilon m)}$: For any $\bx \in G(B_2^k(r))$ and its noiseless measurements $\bb = \Phi(\bx)$, any estimate $\hat{\bx} \in G(B_2^k(r))$ such that $\Phi(\hat{\bx}) = \bb$ satisfies 
     \begin{equation}
      \|\bx - \hat{\bx}\|_2 \le \epsilon.
     \end{equation}
\end{corollary}

In addition, following similar ideas to those in the proof of Theorem~\ref{thm:main_noiseless}, we obtain the following corollary, which provides a supplementary guarantee to that of Theorem~\ref{thm:main_noiseless}. The proof can be found in the supplementary material.

\begin{corollary}\label{coro:main_noiseless}
     Let $\bA$ and $G$ follow the same assumptions as those given in Theorem~\ref{thm:main_noiseless}. Then, for a fixed $\epsilon \in (0,1)$, if $m = \Omega\left(\frac{k}{\epsilon}\log \frac{L r}{\epsilon^2}\right)$, with probability at least $1-e^{-\Omega(\epsilon m)}$, for all $\bx,\bs \in G(B_2^k(r))$, it holds that 
     \begin{equation}
         \|\bx-\bs\|_2 \le \epsilon \Rightarrow \rmd_\rmH(\Phi(\bx),\Phi(\bs)) \le O(\epsilon).
     \end{equation}
\end{corollary}
\begin{remark}
 Combining the results of Theorem~\ref{thm:main_noiseless} and Corollary~\ref{coro:main_noiseless}, we arrive at the so-called Local Binary Embedding property. This property is of independent interest, e.g., see \cite{oymak2015near}, and will also be used as a stepping stone to a stronger binary embedding property in Section~\ref{sec:bese}.  Briefly, the distinction is that the local binary embedding property can be interpreted as ``If $\bx$ is close to $\bs$ then $\Phi(\bx)$ is close to $\Phi(\bs)$ (and vice versa)'', whereas in Section~\ref{sec:bese} we seek a stronger statement of the form ``The distance between $\bx$ and $\bs$ always approximately equals the distance between $\Phi(\bx)$ and $\Phi(\bs)$''.
\end{remark}

\subsubsection{Proof Outline for Theorem~\ref{thm:main_noiseless}}

To prove Theorem~\ref{thm:main_noiseless}, we follow the technique used in~\cite{Bor17} to construct a chain of nets for $G(B_2^k(r))$, and approximate $\bx$ using a point $\bx_0$ in one of the $\epsilon$-nets (and similarly, approximating $\bs$ using $\bs_0$).  We can control various terms consisting of points in the $\epsilon$-nets using probabilistic arguments and the union bound. Before providing a more detailed outline, we state some useful auxiliary results.

\begin{lemma}{\em \hspace{1sp}\cite[Lemma~4.4]{plan2013one}}\label{lem:noiseless_sep}
     Let $\bx,\bs \in \calS^{n-1}$ and assume that $\|\bx- \bs\|_2 \ge \epsilon$ for some $\epsilon > 0$. Let $\ba \sim \calN(\boldsymbol{0},\bI_n)$. Then for $\epsilon_0 = \frac{\epsilon}{12}$, we have 
     \begin{equation}
      \bbP(\langle \ba, \bx \rangle > \epsilon_0, \langle \ba, \bs \rangle < -\epsilon_0) \ge \epsilon_0. 
     \end{equation}
\end{lemma}

This result essentially states that if two unit vectors are far apart, then for a random hyperplane, the probability of a certain level of separation can be lower bounded. In addition, we will use the following concentration inequality.

\begin{lemma}\label{lem:norm_pres}
    {\em \cite[Lemma 1.3]{vempala2005random}}
     Let $\bx \in \bbR^n$, and assume that the entries in $\bA \in \bbR^{m\times n}$ are sampled independently from $\calN(0,1)$. Then, for any $\epsilon \in (0,1)$, we have
     \begin{align} 
          & \bbP\left((1-\epsilon)\|\bx\|_2^2 \le \Big\|\frac{1}{\sqrt{m}}\bA\bx\Big\|_2^2\le (1+\epsilon)\|\bx\|_2^2\right) \nonumber \\ 
          & \hspace*{3.7cm} \ge 1-2e^{-\epsilon^2 (1 - \epsilon) m/4}.
     \end{align}
\end{lemma}

The following definition formally introduces the notion of an $\epsilon$-net, also known as a covering set.

\begin{definition}
    Let $(\calX,d)$ be a metric space, and fix $\epsilon>0$. A subset $S \subseteq \calX$ is said be an {\em $\epsilon$-net} of $\calX$ if, for all $\bx \in \calX$, there exists some $\bs \in S$ such that $d(\bx,\bs) \le \epsilon$.
\end{definition}

With the above auxiliary results in place, the proof of Theorem \ref{thm:main_noiseless} is outlined as follows:
\begin{enumerate}
    \item For a fixed $\delta >0$ and a positive integer $l$, let $M = M_0 \subseteq M_1 \subseteq \ldots \subseteq M_l$ be a chain of nets of $B_2^k(r)$ such that $M_i$ is a $\frac{\delta_i}{L}$-net with $\delta_i = \frac{\delta}{2^i}$. There exists such a chain of nets with \cite[Lemma~5.2]{vershynin2010introduction}
    \begin{equation}
        \log |M_i| \le k \log\frac{4Lr}{\delta_i}. \label{eq:net_size_main}
    \end{equation}
    Then, by the $L$-Lipschitz assumption on $G$, we have for any $i \in [l]$ that $G(M_i)$ is a $\delta_i$-net of $G(B_2^k(r))$. 
    \item For $\bx \in G(B_2^k(r))$, we write $\bx = (\bx - \bx_l) + (\bx_l - \bx_{l-1}) + \ldots + (\bx_1 - \bx_0) + \bx_0$ with $\bx_i \in  G(M_i)$ and $\|\bx- \bx_l\|_2 \le \frac{\delta}{2^l}$, $\|\bx_i - \bx_{i-1}\|_2 \le \frac{\delta}{2^{i-1}}$, $i \in [l]$. The triangle inequality yields $\|\bx - \bx_0\|_2 \le 2\delta$. We apply similar reasoning to a second signal $\bs \in G(B_2^k(r))$ with $\|\bx - \bs\|_2 > \epsilon$, and choose $\delta = O(\epsilon^2)$ sufficiently small so that $\|\bx_0 - \bs_0\|_2 > \frac{\epsilon}{2}$.  This allows us to apply Lemma \ref{lem:noiseless_sep} to get 
    \begin{equation}\label{eq:lemma8Eq_main}
         \bbP\left(\langle \underline{\ba}_i,\bx_0 \rangle > \frac{\epsilon}{24}, \langle \underline{\ba}_i,\bs_0 \rangle < -\frac{\epsilon}{24}\right) \ge \frac{\epsilon}{24}
    \end{equation}
    for any $i \in [m]$, where $\underline{\ba}_i$ is the $i$-th row of $\bA$.  Since the tests are independent, we can use binomial concentration to deduce that at least an $\Omega(\epsilon)$ fraction of the measurements satisfy the condition in \eqref{eq:lemma8Eq_main}, with probability $1 - e^{-\Omega(\epsilon m)}$.  Then, by \eqref{eq:net_size_main} and a union bound, the same holds simultaneously for {\em all} $(\bx',\bs') \in G(M) \times G(M)$ with high probability.
    \item We use the Cauchy-Schwartz inequality and triangle inequality to obtain the following decomposition:
    \begin{align}
         & \frac{1}{m} \sum_{i=1}^m |\langle \underline{\ba}_i, \bx - \bx_0 \rangle| \nonumber \\ 
         & \le \sum_{i=1}^l \Big\|\frac{1}{\sqrt{m}} \bA (\bx_i-\bx_{i-1})\Big\|_2 + \Big\|\frac{1}{\sqrt{m}} \bA (\bx - \bx_l)\Big\|_2, \label{eq:mimic_bd_main}
    \end{align}
    and upper bound the two terms as follows:
    \begin{enumerate}
        \item For the first term, we use Lemma \ref{lem:norm_pres} and a union bound over the signals in the $i$-th and $(i-1)$-th nets to upper bound each summand by $\big(1 + \frac{\epsilon_i}{2}\big)\frac{\delta}{2^{i-1}}$ with high probability, for some $\epsilon_1,\dotsc,\epsilon_l$.  We show that a choice of the form $\epsilon_i^2 = O\big(\epsilon + \frac{ik}{m}\big)$ suffices to take the overall term down to $O(\delta)$.
        \item For the second term, we upper bound the spectral norm of $\bA$ by $2 + \sqrt{\frac{n}{m}}$ with high probability, and show that when this bound holds, $l = O(\log n)$ suffices to bring the overall term down to $O(\delta)$.
    \end{enumerate}
    This argument holds uniformly in $\bx$, and we apply the resulting bound to both signals $\bx,\bs$ under consideration.  The choice $\delta = O(\epsilon^2)$ allows us to deduce that a fraction $1 - O(\epsilon)$ of the measurements satisfy $|\langle \underline{\ba}_i, \bx-\bx_0\rangle| + |\langle \underline{\ba}_i, \bs-\bs_0\rangle| \le O(\epsilon)$.  The implied constant in this fraction of measurements is carefully designed to be smaller than that in the $\Omega(\epsilon)$ fraction of Step 2.
    \item We combine Steps 2 and 3 to show that a fraction $\Omega(\epsilon)$ of the measurements satisfy {\em both} of the conditions therein, and we show that $\mathrm{sign}(\langle \underline{\ba}_i, \bx\rangle) \ne \mathrm{sign}(\langle \underline{\ba}_i, \bs\rangle)$ for every such measurement.  As a result, we find that $\rmd_\rmH(\Phi(\bx),\Phi(\bs)) \ge \Omega(\epsilon)$, as desired.
\end{enumerate}

\subsection{Lower Bound}\label{sec:lb_noiseless}

In this subsection, we address the question of whether the upper bound in Theorem \ref{thm:main_noiseless} can be improved.  To do this, following the approach of \cite{liu2020information}, we consider a specific $L$-Lipschitz generative model, and derive an algorithm-independent lower bound on the number of samples required to accurately recover signals from this model.  This result is formally stated as follows.

\begin{theorem}\label{thm:lb_noiseless2}
     Fix $r > 0$ and $L = \Omega\big(\frac{1}{r}\big)$ with a sufficiently large implied constant, and $\epsilon \in \big(0,\frac{\sqrt 3}{4 \sqrt 2}\big)$.  Then, there exists an $L$-Lipschitz generative model with input domain $B_2^k(r)$ such that for any measurement matrix $\bA$ and decoder producing an estimate $\hat{\bx}$ such that $\sup_{\bx \in G(B_2^k(r))} \|\bx - \hat{\bx}\|_2 \le \epsilon$, it must be the case that $m = \Omega\left(k \log(Lr) + \frac{k}{\epsilon}\right)$.
\end{theorem}

The proof is given in the supplementary material, and is briefly outlined as follows.  We follow the high-level approach from \cite{liu2020information} of choosing a generative model that can produce group-sparse signals, with suitable normalization to ensure that all signals lie on the unit sphere.  Both the $\Omega\left(\frac{k}{\epsilon}\right)$ and $\Omega\left( k \log(Lr) \right)$ lower bounds are established by choosing a {\em hard subset} of signals, and comparing its size to the number of possible output sequences:
\begin{itemize}
    \item For the $\Omega\left(\frac{k}{\epsilon}\right)$ bound, following \cite{acharya2017improved}, we consider packing as many signals as possible onto a unit sphere corresponding to the subspace of an arbitrary single sparsity pattern, and we bound the number of output sequences using a result from \cite{jacques2013robust} on the number of orthants of $\bbR^m$ intersected by a single lower-dimensional subspace.
    \item For the $\Omega\left( k \log(Lr) \right)$ bound, we use the Gilbert-Varshamov bound to show that there exist $e^{\Omega(k\log\frac{n}{k})}$ sequences separated by a constant distance, and trivially upper bound the number of output sequences by $2^m$.  This gives an $m = \Omega\big(k\log\frac{n}{k}\big)$ lower bound, which reduces to $\Omega(k \log (Lr))$ upon calculating the Lipschitz constant of our chosen generative model.
\end{itemize}

\begin{remark}
    In Theorem~\ref{thm:main_noiseless}, the sample complexity derived is $\Omega\left(\frac{k}{\epsilon}\log \frac{L r }{\epsilon^2}\right)$. 
     Comparing with the lower bound provided in Theorem~\ref{thm:lb_noiseless2}, we observe that when $\epsilon = \Theta(1)$ the upper and lower bounds match, and when $\epsilon = o(1)$, they match up to a logarithmic factor in $\frac{Lr}{\epsilon^2}$.
\end{remark}

\begin{remark}
    A recent result in \cite{Flo19} suggests that the presence of separate $\frac{k}{\epsilon}$ and $k \log (Lr)$ terms (as opposed to a combined term such as $\frac{k}{\epsilon} \log(Lr)$) is the correct behavior in certain cases.  Specifically, it is shown that in the case of sparse signals, one can indeed achieve $m = O\big( \frac{k}{\epsilon} + k \log\frac{n}{k} \big)$ by moving beyond i.i.d.~Gaussian measurement matrices.  However, the technique is based on first identifying a superset of the sparse support, and it is unclear what a suitable counterpart would be in the case of general generative models.
\end{remark}

\section{Binary Embeddings and Noisy Measurements}\label{sec:ub_noisy}

Thus far, we have considered recovery guarantees under noiseless measurements.  In this section, we turn to the {\em Binary $\epsilon$-Stable Embedding} (B$\epsilon$SE) property (defined below), which roughly requires the binary measurements to preserve the geometry of signals produced by the generative model.  Similarly to the case of sparse signals \cite{jacques2013robust}, we will see that this permits 1-bit CS recovery guarantees even in the presence of random or adversarial noise.

\begin{definition}
     Let $\epsilon \in (0,1)$. A mapping $\Phi(\cdot)\,:\, \bbR^n \rightarrow \{-1,1\}^m$ is a Binary $\epsilon$-Stable Embedding (B$\epsilon$SE) for vectors in $G(B_2^k(r)) \subseteq \calS^{n-1}$ if 
     \begin{equation}
        \rmd_\rmS(\bx,\bs) - \epsilon \le \rmd_\rmH(\Phi(\bx),\Phi(\bs)) \le \rmd_\rmS(\bx,\bs) + \epsilon
     \end{equation}
    for all $\bx,\bs \in G(B_2^k(r))$, where $\rmd_\rmS$ is the geodesic distance ({\em cf.}, Section \ref{sec:notations}).
\end{definition}

\subsection{Establishing the B$\epsilon$SE Property}\label{sec:bese}

Our main goal in this section is to prove the following theorem, which gives the B$\epsilon$SE property. 

\begin{theorem}\label{thm:main}
     Let $\bA$ and $G$ follow the same assumptions as those given in Theorem~\ref{thm:main_noiseless}. For a fixed $\epsilon  \in (0,1)$, if $m = \Omega\left(\frac{k}{\epsilon^2}\log \frac{Lr}{\epsilon}\right)$, then with probability at least $1-e^{-\Omega(\epsilon^2 m)}$, we have for all $\bx,\bs \in G(B_2^k(r))$ that
     \begin{equation}\label{eq:bese}
      |\rmd_{\rmS}(\bx,\bs) - \rmd_{\rmH}(\Phi(\bx),\Phi(\bs))| \le \epsilon.
     \end{equation}
\end{theorem}
In the proof of Theorem~\ref{thm:main}, we construct an $\epsilon$-net and use $\bx_0,\bs_0$ in the net to approximate $\bx,\bs$. We use the triangle inequality to decompose $|\rmd_{\rmS}(\bx,\bs) - \rmd_{\rmH}(\Phi(\bx),\Phi(\bs))|$ into three terms: $|\rmd_\rmH(\Phi(\bx),\Phi(\bs)) - \rmd_\rmH(\Phi(\bx_0),\Phi(\bs_0))|$, $|\rmd_\rmH(\Phi(\bx_0),\Phi(\bs_0))-\rmd_\rmS(\bx_0,\bs_0)|$ and $|\rmd_\rmS(\bx,\bs) - \rmd_\rmS(\bx_0,\bs_0)|$. We derive an upper bound for the first term by using Corollary~\ref{coro:main_noiseless} to bound $\rmd_\rmH(\Phi(\bx),\Phi(\bx_0))$ and $\rmd_\rmH(\Phi(\bs),\Phi(\bs_0))$.  The second term is upper bounded using a concentration bound from \cite{jacques2013robust} and a union bound for all $(\bx_0,\bs_0)$ pairs in the $\epsilon$-net. The third term is directly upper bounded via the definition of an $\epsilon$-net.

Before formalizing this outline, we introduce the following useful lemmas. 

\begin{lemma}{\em \hspace{1sp}\cite[Lemma~3.2]{goemans1995improved}}\label{lem:first}
     Suppose that $\br$ is drawn uniformly from the unit sphere $\calS^{n-1}$. Then for any two fixed vectors $\bx, \bs \in \calS^{n-1}$, we have 
     \begin{equation}
      \bbP \left(\mathrm{sign}(\langle \bx, \br \rangle) \ne \mathrm{sign}(\langle \bs, \br \rangle)\right) = \rmd_\rmS(\bx,\bs).
     \end{equation}
\end{lemma}

Based on this lemma, the following lemma concerning the geodesic distance and the Hamming distance follows via a concentration argument. 

\begin{lemma}{\em \hspace{1sp}\cite[Lemma~2]{jacques2013robust}}\label{lem:second}
     For fixed $\bx,\bs \in \calS^{n-1}$ and $\epsilon >0$, we have 
     \begin{equation}
      \bbP\left(|\rmd_\rmH(\Phi(\bx),\Phi(\bs)) - \rmd_\rmS(\bx,\bs)|\le \epsilon\right) \ge 1-2e^{-\epsilon^2 m}, 
     \end{equation}
    where the probability is with respect to the generation of $\bA$ with i.i.d.~standard normal entries.
\end{lemma}
Note that in Lemma \ref{lem:second} the vectors $\bx,\bs$ are fixed in advance, before the sample matrix $\bA$ is drawn. In contrast, for Theorem~\ref{thm:main}, we need to consider drawing $\bA$ first and then choosing $\bx,\bs$ arbitrarily. 

In addition, we have the following simple lemma, which states that the Euclidean norm distance and geodesic distance are almost equivalent for vectors on the unit sphere. 

\begin{lemma}\label{lem:simple}
     For any $\bx,\bs \in \calS^{n-1}$, we have 
     \begin{equation}
      \frac{1}{\pi} \|\bx-\bs\|_2 \le \rmd_\rmS(\bx,\bs) \le \frac{1}{2}\|\bx-\bs\|_2.
     \end{equation}
\end{lemma}
\begin{proof}
     Let $d = \|\bx-\bs\|_2$. Using the definition of geodesic distance and $\langle \bx, \bs \rangle = \frac{1}{2}\big( \|\bx\|_2^2 + \|\bs\|_2^2 - \|\bx - \bs\|_2^2 \big)$, we have $\rmd_\rmS(\bx,\bs) = \frac{1}{\pi} \arccos \big(1-\frac{d^2}{2}\big)$. It is straightforward to show that $1-\frac{2}{\pi^2}x^2 \ge \cos x \ge 1-\frac{x^2}{2}$ for any $x \in [0,\pi]$. In addition, letting $a  = \arccos \big(1-\frac{d^2}{2}\big)$, we have $1-\frac{a^2}{2} \le \cos a = 1-\frac{d^2}{2} \le 1-\frac{2}{\pi^2} a^2$, which implies $d \le a \le \frac{\pi}{2}d$. Therefore, $\frac{d}{\pi}\le\rmd_\rmS(\bx,\bs) = \frac{a}{\pi} \le \frac{d}{2}$. 
\end{proof}

We now proceed with the proof of Theorem~\ref{thm:main}. 

\begin{proof}[Proof of Theorem~\ref{thm:main}.]
     Let $M$ be an $\frac{\epsilon}{L}$-net of $B_2^k(r)$ such that $\log |M| \le k \log\frac{4Lr}{\epsilon}$. By the $L$-Lipschitz property of $G$,  $G(M)$ is a $\epsilon$-net of $G(B_2^k(r))$. Hence, for any $\bx, \bs \in G(B_2^k(r)) $, there exist $\bx_0,\bs_0 \in G(M)$ such that 
     \begin{equation}\label{eq:bxbybsbt}
      \|\bx- \bx_0\|_2 \le \epsilon, \quad \|\bs-\bs_0\|_2 \le \epsilon.
     \end{equation}
    By Lemma~\ref{lem:simple}, we have 
    \begin{equation}
     \rmd_\rmS(\bx,\bx_0) \le \frac{\epsilon}{2}, \quad \rmd_\rmS(\bs,\bs_0) \le \frac{\epsilon}{2},
    \end{equation}
    and hence, by the triangle inequality,
    \begin{equation}\label{eq:first}
     \left|\rmd_{\rmS}(\bx,\bs)-\rmd_\rmS(\bx_0,\bs_0)\right| \le \epsilon.
    \end{equation}
    In addition, by Lemma~\ref{lem:second} and the union bound, if we set $m = \Omega\left(\frac{k}{\epsilon^2} \log \frac{Lr}{\epsilon} \right)$, with probability at least $1-|M|^2 e^{-\epsilon^2 m} = 1-e^{-\Omega(\epsilon^2 m)}$, we have 
    \begin{equation}\label{eq:second}
     |\rmd_\rmH(\Phi(\bu),\Phi(\bv)) - \rmd_\rmS(\bu,\bv)|\le \epsilon
    \end{equation}
    for all $(\bu,\bv) \in G(M) \times G(M)$.
    Furthermore, by Corollary~\ref{coro:main_noiseless} and~\eqref{eq:bxbybsbt}, if $m = \Omega\left(\frac{k}{\epsilon}\log \frac{Lr}{\epsilon^2}\right)$,  then with probability at least $1-e^{-\Omega(\epsilon m)}$, we have  
    \begin{equation}\label{eq:Phi_bxbybsbt}
        \rmd_\rmH(\Phi(\bx),\Phi(\bx_0)) \le C\epsilon, \quad \rmd_\rmH(\Phi(\bs),\Phi(\bs_0)) \le C\epsilon,
    \end{equation}
    where $C$ is a positive constant. (Note that the result of Corollary \ref{coro:main_noiseless} holds {\em uniformly} for signals in $G(B_2^k(r))$.)  Using the two upper bounds in \eqref{eq:Phi_bxbybsbt} and applying the triangle inequality in the same way as \eqref{eq:first}, we obtain
    \begin{equation}
        |\rmd_\rmH(\Phi(\bx),\Phi(\bs)) - \rmd_\rmH(\Phi(\bx_0),\Phi(\bs_0))| \le 2C\epsilon. \label{eq:third}
    \end{equation}

    Combining~\eqref{eq:first},~\eqref{eq:second} and~\eqref{eq:third}, we obtain that if $m = \Omega\left(\frac{k}{\epsilon^2} \log \frac{Lr}{\epsilon} + \frac{k}{\epsilon}\log \frac{Lr}{\epsilon^2}\right)$, with probability at least $1-e^{-\Omega(\epsilon^2 m)} -  e^{-\Omega(\epsilon m)}$, the following holds uniformly in $\bx,\bs \in G(B_2^k(r))$:
    \begin{align}
     & |\rmd_\rmH(\Phi(\bx),\Phi(\bs)) - \rmd_\rmS(\bx,\bs)| \nonumber \\
     & \quad\le |\rmd_\rmH(\Phi(\bx),\Phi(\bs)) - \rmd_\rmH(\Phi(\bx_0),\Phi(\bs_0))| \nonumber  \\ 
     & \quad\qquad + |\rmd_\rmH(\Phi(\bx_0),\Phi(\bs_0))-\rmd_\rmS(\bx_0,\bs_0)| \nonumber \\
     & \quad\qquad + |\rmd_\rmS(\bx,\bs) - \rmd_\rmS(\bx_0,\bs_0)| \nonumber \\
     & \quad\le 2C\epsilon + \epsilon + \epsilon \nonumber \\
     &= 2(C+1)\epsilon. \label{eq:combine3}
    \end{align}
    Then, recalling that $Lr = \Omega(1)$, and scaling $\epsilon$ by $2(C+1)$, we deduce that when $m = \Omega\left(\frac{k}{\epsilon^2} \log \frac{Lr}{\epsilon}\right)$, we have with probability at least $1-e^{-\Omega(\epsilon^2 m)}$ that 
    \begin{equation}
     |\rmd_\rmH(\Phi(\bx),\Phi(\bs)) - \rmd_\rmS(\bx,\bs)| \le \epsilon.
    \end{equation}
\end{proof}

\subsection{Implications for Noisy 1-bit CS}

Here we demonstrate that Theorem~\ref{thm:main} implies recovery guarantees for 1-bit CS in the case of noisy measurements. In particular, we have the following corollary. 
\begin{corollary}\label{coro:noisy_exp}
    Let $\bA$ and $G$ follow the same assumptions as those given in Theorem~\ref{thm:main_noiseless}. For an $\epsilon  \in (0,1)$, if $m = \Omega\left(\frac{k}{\epsilon^2}\log \frac{Lr}{\epsilon}\right)$, then with probability at least $1-e^{-\Omega(\epsilon^2 m)}$, we have the following: For any $\bx \in G(B_2^k(r))$, if $\tilde{\bb} := \mathrm{sign}(\bA \bx)$ and $\bb$ is any vector of corrupted measurements satisfying $\rmd_\rmH(\bb,\tilde{\bb})\le \tau_1$, then any $\hat{\bx} \in G(B_2^k(r))$ with $\rmd_\rmH(\mathrm{sign}(\bA\hat{\bx}),\bb) \le \tau_2$ satisfies
    \begin{equation}
        \rmd_{\rmS}\left(\bx,\hat{\bx}\right) \le \epsilon + \tau_1 + \tau_2. \label{eq:approx_dS}
    \end{equation}
\end{corollary}
\begin{proof}
 By Theorem~\ref{thm:main}, we have
 \begin{equation}\label{eq:bese_coro}
  |\rmd_{\rmS}(\bx,\hat{\bx}) - \rmd_{\rmH}(\Phi(\bx),\Phi(\hat{\bx}))| \le \epsilon.
 \end{equation}
 In addition, by the triangle inequality and $\tilde{\bb} = \Phi(\bs)$,
 \begin{equation}\label{eq:triangle_coro}
  \rmd_{\rmH}(\Phi(\bx),\Phi(\hat{\bx})) \le \rmd_\rmH(\bb,\tilde{\bb})+ \rmd_\rmH(\mathrm{sign}(\bA\hat{\bx}),\bb) \le \tau_1 + \tau_2.
 \end{equation}
 Combining~\eqref{eq:bese_coro}--\eqref{eq:triangle_coro} gives \eqref{eq:approx_dS}. 
\end{proof} 

Corollary~\ref{coro:noisy_exp} gives a guarantee for arbitrary (possibly adversarial) perturbations of $\tilde{\bb}$ to produce $\bb$.  Naturally, this directly implies high-probability bounds on the recovery error in the case of random noise. For instance, if $\bb = \mathrm{sign}(\bA\bx + \bxi)$ with $\bxi \sim \calN(\boldsymbol{0},\sigma^2 \bI_m)$, then by~\cite[Lemma 4]{jacques2013robust}, we have for any $\gamma > 0$ that
\begin{equation}
 \bbP \left[\rmd_\rmH(\tilde{\bb},\bb) > \frac{\sigma}{2} + \gamma\right] \le e^{-2m\gamma^2}.
\end{equation}  
Analogous results can be derived for other noise distributions such as Poisson noise, random sign flips, and so on. In addition, we may derive upper bounds on $\rmd_\rmH(\mathrm{sign}(\bA\hat{\bx}),\bb)$ for specific algorithms. For example, for algorithms with consistent sign constraints, we have $\rmd_\rmH(\mathrm{sign}(\bA\hat{\bx}),\bb) = 0$, which corresponds to $\tau_2 = 0$ in Corollary~\ref{coro:noisy_exp}.

\section{Neural Network Generative Models}\label{sec:nn_apply}

In this section, we consider feedforward neural network generative models. Such a model $\tilde{G}: \bbR^{k} \rightarrow  \bbR^n$ with $d$ layers can be written as 
\begin{equation}\label{eq:nn_function}
 \tilde{G}(\bz) = \phi_d\left(\phi_{d-1}\left(\cdots \phi_2( \phi_1(\bz,\btheta_1), \btheta_2)\cdots, \btheta_{d-1}\right), \btheta_d\right),
\end{equation}
where $\bz \in B_2^k(r)$, $\phi_i(\cdot)$ is the functional mapping corresponding to the $i$-th layer, and $\btheta_i = (\bW_i,\bb_i)$ is the parameter pair for the $i$-th layer:  $\bW_i \in \bbR^{n_i \times n_{i-1}}$ is the matrix of weights, and $\bb_i \in \bbR^{n_i}$ is the vector of offsets, where $n_i$ is the number of neurons in the $i$-th layer.  Note that $n_0 = k$ and $n_d = n$. Defining $\bz^0 = \bz$ and $\bz^i= \phi_i(\bz^{i-1},\btheta_i)$, we set $\phi_i(\bz^{i-1},\btheta_i) = \phi_i(\bW_i \bz^{i-1}+ \bb_i)$, $i = 1,2,\ldots,d$, for some operation $\phi_i(\cdot)$ applied element-wise. 

The function $\phi_i(\cdot)$ is referred to as the activation function for the $i$-th layer, with popular choices including (i) the ReLU function, $\phi_i(x) = \max(x,0)$; (ii) the Sigmoid function, $\phi_i(x) = \frac{1}{1+e^{-x}}$; and (iii) the Hyperbolic tangent function with $\phi_i(x) = \frac{e^{x}-e^{-x}}{e^{x}+e^{-x}}$. Note that for each of these examples, $\phi_i(\cdot)$ is $1$-Lipschitz.

To establish Lipschitz continuity of the entire network, we can utilize the following standard result.
\begin{lemma}\label{lem:lip_comp}
    Consider any two functions $f$ and $g$. If $f$ is $L_f$-Lipschitz and $g$ is $L_g$-Lipschitz, then their composition $f \circ g$ is $L_f L_g$-Lipschitz.
\end{lemma}

Suppose that $\tilde{G}$ is defined as in~\eqref{eq:nn_function} with at most $w$ nodes per layer. We assume that all weights are upper bounded by $W_{\max}$ in absolute value, and that the activation functions are $1$-Lipschitz. Then, from Lemma~\ref{lem:lip_comp}, we obtain that $\tilde{G}$ is $\tilde{L}$-Lipschitz with $\tilde{L} = \left(w W_{\max}\right)^d$ ({\em cf.}~\cite[Lemma~8.5]{Bor17}).  Since we consider normalized signals having unit norm, we limit our attention to signals in ${\rm Range}(\tilde{G})$ with norm at least $R_{\min}$, for some small $R_{\min} > 0$, so as to control the Lipschitz constant of $\frac{\tilde{G}(\bz)}{\|\tilde{G}(\bz)\|_2}$.  We obtain the following from Theorem~\ref{thm:main}. 



\begin{theorem}\label{thm:main_nn}
     Suppose that $\bA \in \bbR^{m\times n}$ is generated with $A_{ij} \overset{i.i.d.}{\sim} \calN(0,1)$, and the generative model $\tilde{G}: B_2^k(r) \rightarrow \bbR^n$ is defined as in~\eqref{eq:nn_function} with at most $w$ nodes per layer. Suppose that all weights are upper bounded by $W_{\max}$ in absolute value, and that the activation function is $1$-Lipschitz. 
     Then, for fixed $\epsilon  \in (0,1)$ and $R_{\rm min} > 0$, if $m = \Omega\big(\frac{k}{\epsilon^2}\log \frac{r(wW_{\max})^d }{\epsilon R_{\min}}\big)$, with probability at least $1-e^{-\Omega(\epsilon^2 m)}$, we have the following: For any $\bx \in \tilde{G}(B_2^k(r)) \setminus B_2^n(R_{\rm min})$, let $\tilde{\bb} := \mathrm{sign}(\bA \bx)$ be its uncorrupted measurements, and let $\bb$ be any corrupted measurements satisfying $\rmd_\rmH(\bb,\tilde{\bb})\le \tau_1$. Then, any $\hat{\bx} \in \tilde{G}(B_2^k(r)) \setminus B_2^n(R_{\rm min})$ with $\rmd_\rmH(\mathrm{sign}(\bA\hat{\bx}),\bb) \le \tau_2$ satisfies
     \begin{equation}
      \rmd_{\rmS}\left(\frac{\bx}{\|\bx\|_2},\frac{\hat{\bx}}{\|\hat{\bx}\|_2}\right) \le \epsilon  + \tau_1 + \tau_2.
     \end{equation}
\end{theorem}
\begin{proof}
    Let $\tilde{\calD}= \{\bz \in B_2^k(r)\,:\, \|\tilde{G}(\bz)\|_2 > R_{\min}\}$, and define $G(\bz) := \frac{\tilde{G}(\bz)}{\|\tilde{G}(\bz)\|_2}$ for $\bz \in \tilde{\calD}$. Observe that $G(\tilde{\calD}) \subseteq \calS^{n-1}$.  In addition, by Lemma~\ref{lem:lip_comp} and the assumption $\|\tilde{G}(\bz)\|_2 \ge R_{\min}$, we have that $G$ is $L$-Lipschitz on $\tilde{\calD}$ with $L = \frac{(w W_{\max})^d}{R_{\min}}$.  Recall that Theorem~\ref{thm:main} and Corollary \ref{coro:noisy_exp} are proved by forming an $\epsilon$-net of $B_2^k(r)$.  Since an $\epsilon$-net for a given set implies a $2\epsilon$-net of the same size (or smaller) for any subset, the same results hold (with a near-identical proof) when the domain of the generative model is restricted to $\tilde{\calD} \subseteq B_2^k(r)$.  Thus, the desired result follows by applying Corollary~\ref{coro:noisy_exp} to $G$ with the restricted domain $\tilde{\calD}$. 
\end{proof}
\begin{remark}
    The dependence on $R_{\min}$ in the sample complexity is very mild; for instance, under the typical scaling of $(w W_{\max})^d = n^{O(d)}$ \cite{Bor17}, the scaling laws remain unchanged even with $R_{\min} = \frac{1}{n^{O(d)}}$. In addition, for common types of data such as images and audio, vectors with a very low norm are not of significant practical interest (e.g., a flat audio signal or an image of all black pixels).
\end{remark}

\section{Efficient Algorithm \& Numerical Example} \label{sec:algo}

In \cite{jacques2013robust}, an algorithm termed Binary Iterative Hard Thresholding (BIHT)  was proposed for 1-bit compressive sensing of sparse signals.  In the case of a generative prior, we can adapt the BIHT algorithm by replacing the hard thresholding step by a projection onto the generative model.  This gives the following iterative procedure:
\begin{equation}
    \bx^{(t+1)}  = \calP_G \left(\bx^{(t)} + \lambda \bA^T(\bb - \mathrm{sign}(\bA\bx^{(t)}))\right), \label{eq:pgd}
\end{equation}
where $\calP_G(\cdot)$ is the projection function onto $G(B_2^k(r))$, $\bb$ is the observed vector, $\bx^{(0)} = \mathbf{0}$, and $\lambda >0$ is a parameter.  A counterpart to \eqref{eq:pgd} for the linear model was also recently proposed in \cite{Sha18}.

It has been shown in~\cite{jacques2013robust} that the quantity $\bA^T(\mathrm{sign}(\bA\bx)-\bb)$ is a subgradient of the convex one-sided $\ell_1$-norm $2\|[\bb \odot (\bA\bx)]_{-}\|_1$, where ``$\odot$'' denotes the element-wise product and $[x]_{-} = \min\{x,0\}$. Therefore, the BIHT algorithm can be viewed as a projected gradient descent (PGD) algorithm that attempts to minimize $\|[\bb \odot (\bA\bx)]_{-}\|_1$. In addition, as argued in~\cite{jacques2013quantized}, there exist certain promising properties suggesting the stability and convergence of the BIHT algorithm. 

{\bf Numerical example.} While our main contributions are theoretical, in the following we present a simple proof-of-concept experiment for the MNIST dataset.  The dataset consists of 60,000 handwritten images, each of size 28x28 pixels. The variational autoencoder (VAE) model uses a pre-trained VAE with a latent dimension of $k=20$. The encoder and decoder both have the structure of a fully connected neural network with two hidden layers.

The projection step in \eqref{eq:pgd} is approximated using gradient descent, performed using the Adam optimizer with $200$ steps and a learning rate of $0.1$.  The update of $\bx^{(t)}$ in \eqref{eq:pgd} is done with a step size of $\lambda = 1.25$, with a total of 15 iterations.  To reduce the impact of local minima, we choose the best estimate among $4$ random restarts.  The reconstruction error is calculated over 10 images by averaging the per-pixel error in terms of the $\ell_2$-norm.  In accordance with our theoretical results, we focus on i.i.d.~Gaussian measurements.

\begin{figure}[!tbp]
  \centering
  \begin{minipage}[b]{0.475\textwidth}
    \includegraphics[width=\textwidth]{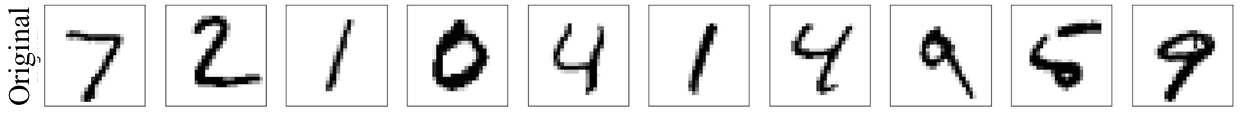}
    \includegraphics[width=\textwidth]{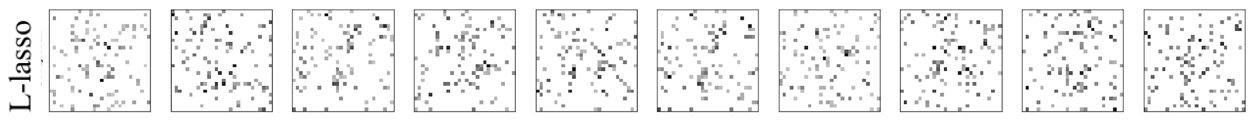}
    \includegraphics[width=\textwidth]{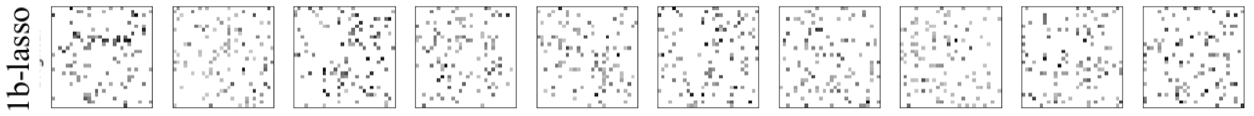}
    \includegraphics[width=\textwidth]{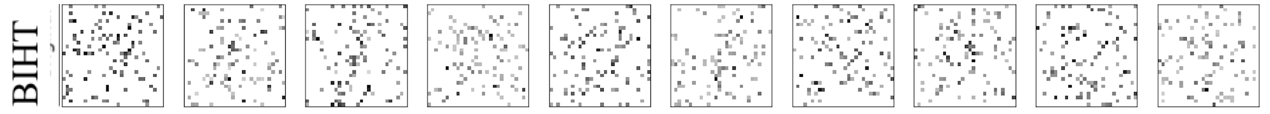}
    \includegraphics[width=\textwidth]{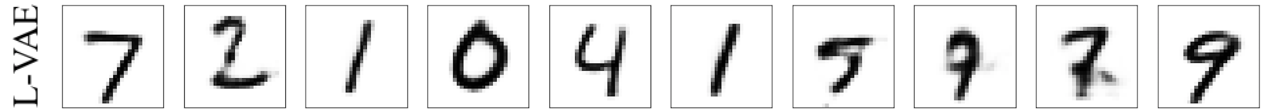}
    \includegraphics[width=\textwidth]{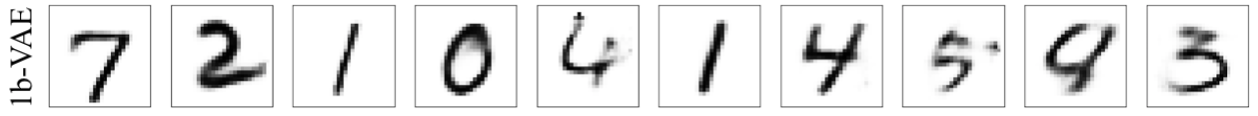}
  \end{minipage} \quad
  \begin{minipage}[b]{0.475\textwidth}
    \includegraphics[width=\textwidth]{Original.png}
    \includegraphics[width=\textwidth]{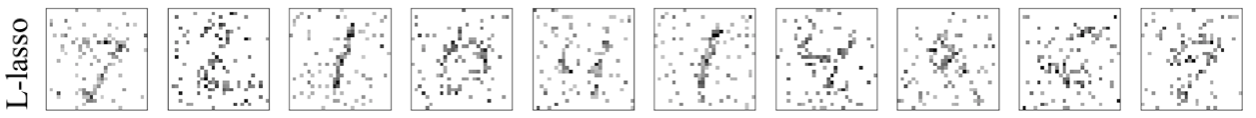}
    \includegraphics[width=\textwidth]{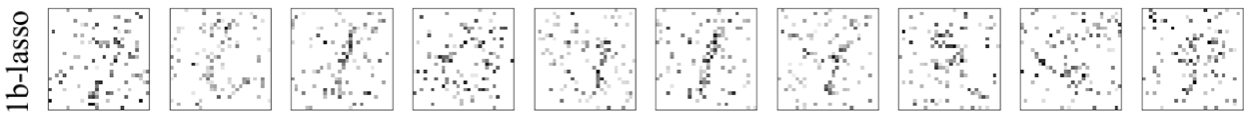}
    \includegraphics[width=\textwidth]{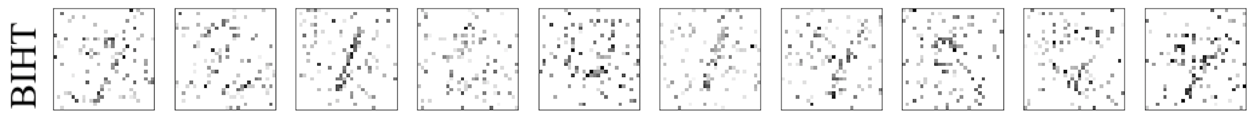}
    \includegraphics[width=\textwidth]{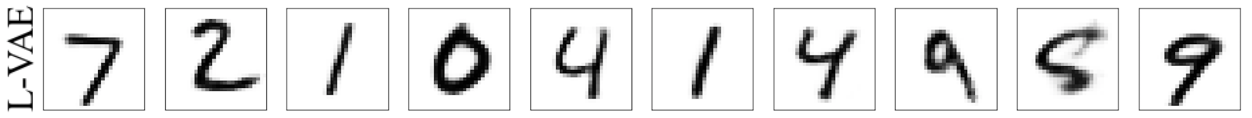}
    \includegraphics[width=\textwidth]{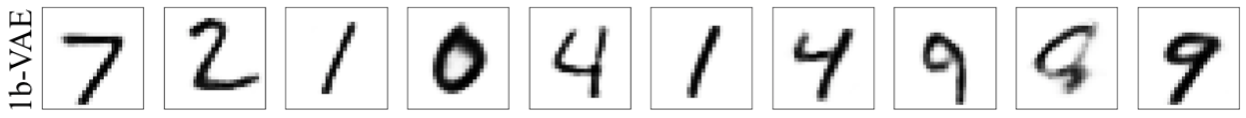}
  \end{minipage}
\caption{Examples of reconstructed images with 50 measurements (top) and reconstruction with 200 measurements (bottom) on the MNIST dataset. \label{fig:recon}}
   \vspace*{-2ex}
\end{figure}

\begin{figure}[!tbp]
  \centering
  \includegraphics[width=0.5\columnwidth]{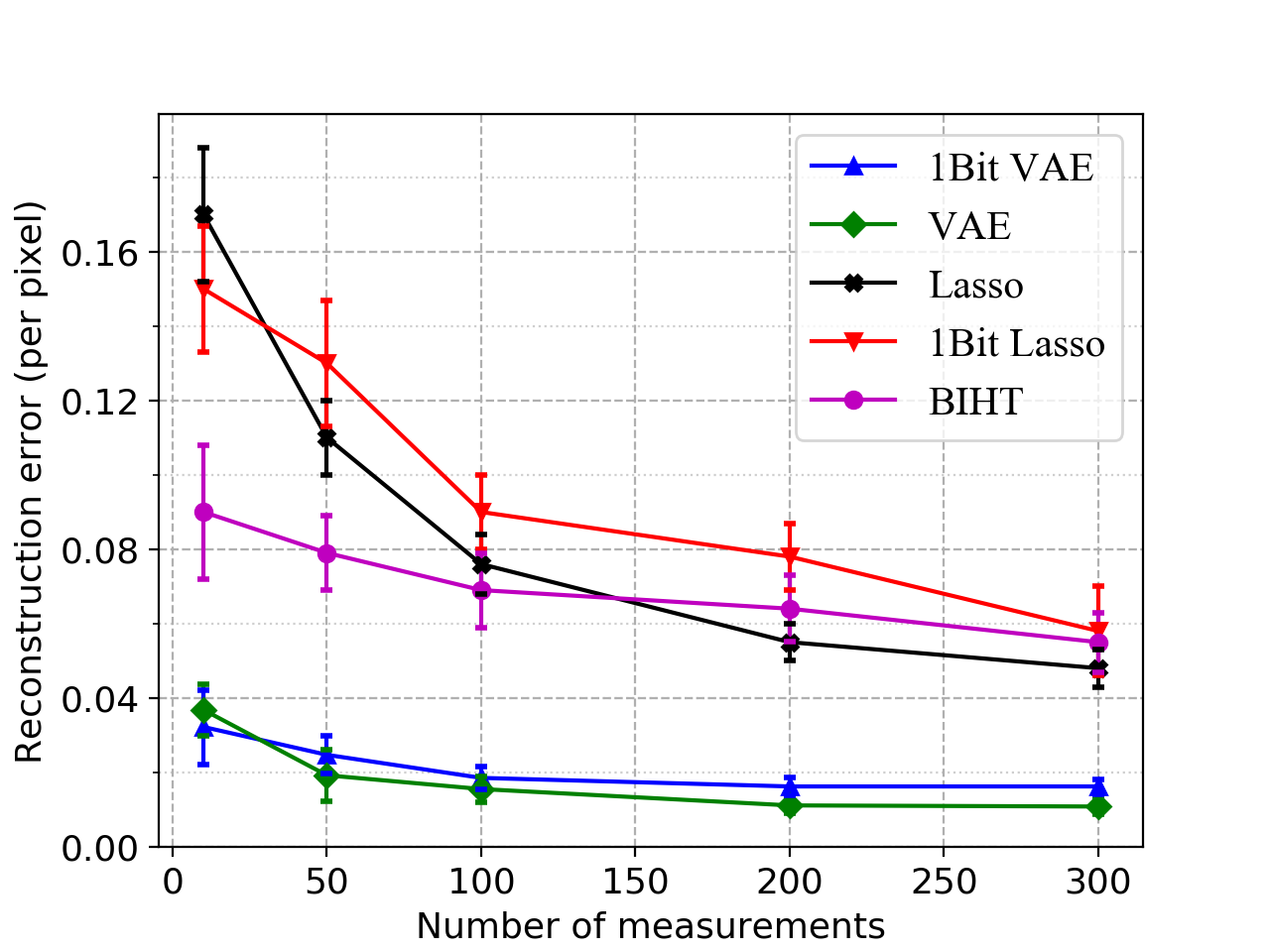}
  \caption{Average reconstruction error (per pixel) of the images from the MNIST dataset shown in Figure \ref{fig:recon}.  The error bars indicate half of a standard deviation.  \label{fig:error}}
   \vspace*{-2ex}
\end{figure}

In Figure \ref{fig:recon}, we provide some examples of reconstructed images under both linear measurements and 1-bit measurements,\footnote{We used the PGD implementation of \cite{Sha18} available at \url{https://github.com/shahviraj/pgdgan}, along with the pre-trained generative model and Lasso implementation of \cite{Bor17} available at \url{https://github.com/AshishBora/csgm}, and adapted these to their 1-bit variants.} using sparsity based algorithms (Lasso \cite{Tib96}, 1-bit Lasso \cite{plan2013one}, and BIHT \cite{jacques2013robust}) and generative prior based algorithms (linear PGD \cite{Sha18} and 1-bit PGD as per \eqref{eq:pgd}).  For convenience, we re-state the Lasso and 1-bit Lasso optimization problems here: We solve
\begin{equation}
    {\rm minimize}~~ \|\bx\|_1 \text{ ~~s.t.~~ } \bb = \bA\bx
\end{equation}
for linear measurements, and we solve
\begin{equation}
    {\rm minimize}~~ \|\bx\|_1 \text{ ~~s.t.~~ } \bb = {\rm sign}(\bA\bx), \|\bA\bx\|_1 = m
\end{equation}
for 1-bit measurements.  As discussed in \cite{plan2013one}, the second constraint can be viewed as normalization that prevents a zero or near-zero solution.

We observe from Figure \ref{fig:recon} that all three sparsity-based methods attain poor reconstructions even when $m=200$.  In contrast, the generative prior based methods attain mostly accurate reconstructions even when $m=50$, and highly accurate constructions when $m=200$.  

In this experiment, the loss due to the 1-bit quantization appears to be mild, and this is corroborated in Figure \ref{fig:error}, where we plot the average per-pixel reconstruction error as a function of the number of measurements, averaged over the images show in Figure \ref{fig:recon}.  For both generative model based methods, the reconstruction error eventually saturates, most likely due to {\em representation error} (i.e., the generative model being unable to perfectly produce the signal) \cite{Bor17}.  In addition, the curve for 1-bit measurements saturates to a slightly higher value than that of linear measurements, most likely due to the impossibility of estimating the norm.  However, at least for this particular dataset, the gap between the two remains small.

\section{Conclusion}

We have established sample complexity bounds for both noiseless and noisy 1-bit compressive sensing with generative models.  In the noiseless case, we showed that the sample complexity for $\epsilon$-accurate recovery is between $\Omega\left(k \log(Lr) + \frac{k}{\epsilon}\right)$ and $O\big( \frac{k}{\epsilon}\log \frac{L r }{\epsilon^2} \big)$.  For noisy measurements, we showed that the binary $\epsilon$-stable embedding property can be attained with $m = O\left(\frac{k}{\epsilon^2}\log \frac{Lr}{\epsilon}\right)$.  An immediate direction for further research is to establish a matching lower bound for the latter result, though we are unaware of any such result even for the simpler case of sparse signals. 

\appendices

\section{Proof of Theorem \ref{thm:main_noiseless} (Noiseless Upper Bound)}

    For fixed $\delta >0$ and a positive integer $l$, let $M = M_0 \subseteq M_1 \subseteq \ldots \subseteq M_l$ be a chain of nets of $B_2^k(r)$ such that $M_i$ is a $\frac{\delta_i}{L}$-net with $\delta_i = \frac{\delta}{2^i}$.  There exists such a chain of nets with \cite[Lemma~5.2]{vershynin2010introduction}
    \begin{equation}
        \log |M_i| \le k \log\frac{4Lr}{\delta_i}. \label{eq:net_size}
    \end{equation}
    By the $L$-Lipschitz assumption on $G$, we have for any $i \in [l]$ that $G(M_i)$ is a $\delta_i$-net of $G(B_2^k(r))$. 

    For any pair of points $\bx,\bs \in G(B_2^k(r))$ with $\|\bx-\bs\|_2 > \epsilon$, we write 
    \begin{align}
         \bx & = (\bx - \bx_l) + (\bx_l - \bx_{l-1}) + \ldots + (\bx_1 - \bx_0) + \bx_0, \\
         \bs& =  (\bs - \bs_l) + (\bs_l - \bs_{l-1}) + \ldots + (\bs_1 - \bs_0) + \bs_0,
    \end{align}
    where $\bx_i, \bs_i \in G(M_i)$ for all $i \in [l]$, and $\|\bx- \bx_l\| \le \frac{\delta}{2^l}$, $\|\bs- \bs_l\| \le \frac{\delta}{2^l}$, $\|\bx_i - \bx_{i-1}\|_2 \le \frac{\delta}{2^{i-1}}$, and $\|\bs_i - \bs_{i-1}\|_2 \le \frac{\delta}{2^{i-1}}$ for all $i \in [l]$. Therefore, the triangle inequality gives
    \begin{equation}\label{eq:bxbxbsbt}
        \|\bx-\bx_0\|_2 < 2\delta, \quad \|\bs-\bs_0\|_2 < 2\delta.
    \end{equation}
    Let $\delta = c_1\epsilon^2$ with $c_1>0$ being a sufficiently small constant. From \eqref{eq:bxbxbsbt}, the triangle inequality, and $\|\bx-\bs\|_2 > \epsilon$, we obtain 
    \begin{equation}
         \|\bx_0-\bs_0\|_2 > \frac{\epsilon}{2}.
    \end{equation}
    This separation between $\bx_0$ and $\bs_0$ permits the application of Lemma~\ref{lem:noiseless_sep}.  Specifically, letting $\underline{\ba}_i \in \bbR^n$ be the $i$-th row of $\bA$, Lemma~\ref{lem:noiseless_sep} (with $\frac{\epsilon}{2}$ in place of $\epsilon$) implies for each $i \in [m]$ that
    \begin{equation}\label{eq:lemma8Eq}
         \bbP\left(\langle \underline{\ba}_i,\bx_0 \rangle > \frac{\epsilon}{24}, \langle \underline{\ba}_i,\bs_0 \rangle < -\frac{\epsilon}{24}\right) \ge \frac{\epsilon}{24}.
    \end{equation}
    In accordance with the event inside the probability, and adopting the generic notation $(\bx',\bs') \in G(M) \times G(M)$ for an arbitrary pair in the net with $\|\bx'-\bs'\|_2>\frac{\epsilon}{2}$, we define 
     \begin{equation}
          \tilde{I}(\bx',\bs') := \bigg\{i \in [m] \,:\, \langle \underline{\ba}_i,\bx' \rangle > \frac{\epsilon}{24}, \langle \underline{\ba}_i,\bs' \rangle < -\frac{\epsilon}{24} \bigg\}. \label{eq:Itil}
     \end{equation}
     By~\eqref{eq:lemma8Eq} and a standard concentration inequality for binomial random variables~\cite[Theorem.~A.1.13]{alon2004probabilistic}, we have 
     \begin{equation}
        \bbP\left(|\tilde{I}(\bx',\bs')| < \frac{\epsilon m}{48}\right) \le e^{-\frac{\epsilon m}{192}}. 
     \end{equation}
    Recall from \eqref{eq:net_size} that $\log |M| \le k \log\frac{4Lr}{\delta}$.  By the union bound, for $m = \Omega\left(\frac{k}{\epsilon}\log \frac{L r}{\delta}\right)$, we have with probability at least $1-\exp(-\Omega(\epsilon m))$ that for {\em all} $(\bx',\bs') \in G(M) \times G(M)$ with $\|\bx'-\bs'\|_2 > \frac{\epsilon}{2}$, the following holds:
    \begin{equation}\label{eq:num_lb}
         |\tilde{I}(\bx',\bs')| \ge \frac{\epsilon m}{48}.
    \end{equation}
    
    We now turn to bounding the following normalized summation:
    \begin{align}
         \frac{1}{m} \sum_{i=1}^m |\langle \underline{\ba}_i, \bx - \bx_0 \rangle|
         & \le \left(\frac{1}{m}\sum_{i=1}^m \langle \underline{\ba}_i, \bx-\bx_0\rangle^2 \right)^{1/2} \\
         & = \Big\|\frac{1}{\sqrt{m}} \bA (\bx-\bx_0)\Big\|_2 \\ 
         & = \Big\|\frac{1}{\sqrt{m}} \bA \left(\sum_{i=1}^l (\bx_i-\bx_{i-1})\right) +  \frac{1}{\sqrt{m}} \bA (\bx - \bx_l)\Big\|_2 \\
         & \le \sum_{i=1}^l \Big\|\frac{1}{\sqrt{m}} \bA (\bx_i-\bx_{i-1})\Big\|_2 + \Big\|\frac{1}{\sqrt{m}} \bA (\bx - \bx_l)\Big\|_2. \label{eq:mimic_bd}
    \end{align}
    Using $\sqrt{1+\varepsilon} \le 1+\frac{\varepsilon}{2}$ (for $\varepsilon \ge -1$), Lemma~\ref{lem:norm_pres}, and the union bound, we have that for any $\epsilon_1,\ldots, \epsilon_l \in (0,1)$, with probability at least $1-\sum_{i=1}^l |M_i|\times |M_{i-1}| \times e^{-\Omega(\epsilon_i^2 m)}$, the following holds for all $i \in [l]$:
    \begin{equation}
         \Big\|\frac{1}{\sqrt{m}} \bA (\bx_i-\bx_{i-1})\Big\|_2 \le \left(1+\frac{\epsilon_i}{2}\right) \|\bx_i-\bx_{i-1}\|_2.  \label{eq:l2_bound1}
    \end{equation}
    uniformly in $\bx_i \in G(M_i)$ and $\bx_{i-1} \in G(M_{i-1})$.  In addition, \eqref{eq:net_size} gives $\log \left( |M_i|\times |M_{i-1}|\right) \le 2i k + 2k\log\frac{4Lr}{\delta}$. As a result, if we choose the $\epsilon_i$ to satisfy $\epsilon_i^2 = \Theta(\epsilon + \frac{ik}{m})$, then we have
    \begin{align}
         \sum_{i=1}^l |M_i|\times |M_{i-1}| \times e^{-\Omega(\epsilon_i^2 m)} & \le e^{-\Omega(\epsilon m)} \sum_{i=1}^l e^{-c_2 ik} \\
         & = e^{-\Omega(\epsilon m)},
    \end{align}
    where $c_2$ is a positive constant.  

    Recall that $m = \Omega\left(\frac{k}{\epsilon}\log \frac{L r}{\delta}\right)$, and that we assume $L = \Omega\big(\frac{1}{r}\big)$ with a sufficiently large implied constant; these together imply $m = \Omega\big( \frac{k}{\epsilon} \big)$.  Hence, we have
    \begin{align}
         \sum_{i=1}^l \Big\|\frac{1}{\sqrt{m}} \bA (\bx_i-\bx_{i-1})\Big\|_2 & \le \sum_{i=1}^l \left(1+\frac{\epsilon_i}{2}\right)\|\bx_i-\bx_{i-1}\|_2 \label{eq:l2_step1} \\
        & \le \sum_{i=1}^l \left(1+\frac{\epsilon_i}{2}\right) \frac{\delta}{2^{i-1}} \label{eq:l2_step2} \\
        & \le \delta \sum_{i=1}^l \frac{\sqrt{\epsilon}}{2^{i-1}} \times O\Big( \sqrt{1 + \frac{ik}{m\epsilon}}\Big) \label{eq:l2_step3} \\
        & = O(\sqrt{\epsilon}\delta) \label{eq:l2_step4a} \\
        & = O(\delta), \label{eq:l2_step4}
    \end{align}
    where \eqref{eq:l2_step1} follows from \eqref{eq:l2_bound1}, \eqref{eq:l2_step2} uses the definition of $\bx_i$, \eqref{eq:l2_step3} follows from the above choice of $\epsilon_i$, and \eqref{eq:l2_step4a} from the above-established fact $m = \Omega\big( \frac{k}{\epsilon} \big)$, and \eqref{eq:l2_step4} since we selected $\delta = c_1 \epsilon^2$.

    Recall that $\|\cdot\|_{2 \to 2}$ is the spectral norm.  By \cite[Corollary~5.35]{vershynin2010introduction}, we have $\big\|\frac{1}{\sqrt{m}} \bA\big\|_{2 \to 2} \le 2+ \sqrt{\frac{n}{m}}$ with probability at least $1-e^{-m/2}$. Hence, choosing $l = \lceil \log_2 n \rceil$, we have with probability at least $1-e^{-m/2}$ that
    \begin{equation}
         \Big\|\frac{1}{\sqrt{m}} \bA (\bx - \bx_l)\Big\|_{2 \to 2} \le \left(2+ \sqrt{\frac{n}{m}}\right) \frac{\delta}{2^l} = O(\delta), \label{eq:above2}
    \end{equation}
    where we used the fact that $\|\bx - \bx_l\| \le \frac{\delta}{2^i}$.
    
    Substituting \eqref{eq:l2_step4} and \eqref{eq:above2} into \eqref{eq:mimic_bd}, we deduce that with probability at least $1-e^{-\Omega(\epsilon m)}$,
    \begin{equation}
        \frac{1}{m} \sum_{i=1}^m |\langle \underline{\ba}_i, \bx - \bx_0 \rangle| \le c_3 \delta,
    \end{equation}
    where $c_3 >0$ is a constant.  Note that this holds {\em uniformly} in $\bx \in G(B_2^k(r))$, since all preceding high-probability events only concerned signals in the chain $M_0,\dotsc,M_l$ of nets, and were proved uniformly with respect to those nets.  Taking the inequality for both $\bx$ and $\bs$ and adding the two together, we obtain 
    \begin{equation}\label{eq:noise_vector_bds}
         \frac{1}{m} \sum_{i=1}^m |\langle \underline{\ba}_i, \bx-\bx_0\rangle| + \frac{1}{m} \sum_{i=1}^m |\langle \underline{\ba}_i, \bs-\bs_0 \rangle| \le 2 c_3\delta.
    \end{equation}

    To combine the preceding findings, let $I_1 = \tilde{I}(\bx_0,\bs_0)$ ({\em cf.}, \eqref{eq:Itil}), and 
    \begin{equation}\label{eq:I2_def}
         I_2 = \left\{i \in [m] : |\langle \underline{\ba}_i, \bx-\bx_0\rangle| + |\langle \underline{\ba}_i, \bs-\bs_0\rangle| \le \frac{192 c_3 \delta}{\epsilon}\right\}.
    \end{equation}
    By~\eqref{eq:num_lb} and~\eqref{eq:noise_vector_bds}, we have that when $m = \Omega\left(\frac{k}{\epsilon}\log \frac{L r}{\delta}\right) = \Omega\left(\frac{k}{\epsilon}\log \frac{L r}{\epsilon^2}\right)$ (recalling the choice $\delta = c_1 \epsilon^2$), with probability at least $1-\exp(-\Omega(\epsilon m))$,
    \begin{equation}\label{eq:bds_I1I2}
        |I_1| \ge \frac{\epsilon m}{48}, \quad |I_2^c| \le \frac{\epsilon m}{96}.
    \end{equation}
    Defining $I := I_1 \cap I_2$, it follows that 
    \begin{equation}\label{eq:lb_cardI}
         |I| \ge |I_1| - |I_2^c| \ge \frac{\epsilon m}{96}.
    \end{equation}
    In addition, for any $i \in I$, we have 
    \begin{align}
         \langle \underline{\ba}_i, \bx \rangle  & =  \langle \underline{\ba}_i, \bx_0\rangle + \langle \underline{\ba}_i, \bx -\bx_0\rangle \\
         & > \frac{\epsilon}{24} - \frac{192c_3\delta}{\epsilon} \label{eq:recallDelta}\\
         & = \frac{\epsilon}{24} -  192 c_1 c_3 \epsilon \label{eq:c1_small1} \\
        & > \frac{\epsilon}{25}, \label{eq:c1_small}
    \end{align}
    where~\eqref{eq:c1_small1} holds because $\delta =c_1 \epsilon^2$, and \eqref{eq:c1_small} follows by choosing $c_1$ sufficiently small.   By a similar argument, we have for $i \in I$ that $\langle \underline{\ba}_i, \bs\rangle < -\frac{\epsilon}{25}$. 
    Therefore, for any $i \in I$, we have $1 = \mathrm{sign}(\langle \underline{\ba}_i, \bx\rangle) \ne \mathrm{sign}(\langle \underline{\ba}_i, \bs\rangle) = -1$, and \eqref{eq:lb_cardI} gives
    \begin{equation}
        \rmd_\rmH(\Phi(\bx),\Phi(\bs)) \ge \frac{|I|}{m} \ge \frac{\epsilon}{96},
    \end{equation}
    which leads to the desired result in Theorem \ref{thm:main_noiseless}.

\section{Proof of Corollary~\ref{coro:main_noiseless} (Supplementary Guarantee to Theorem \ref{thm:main_noiseless})}  

Similar to Lemma~\ref{lem:noiseless_sep}, we have the following lemma.

\begin{lemma}\label{lem:noiseless_sep2}
     Let $\bx,\bs \in \calS^{n-1}$ and assume that $\|\bx- \bs\|_2 \le \epsilon$ for some $\epsilon > 0$. If $\ba \sim \calN(\boldsymbol{0},\bI_n)$, then for $\epsilon_0 = \frac{\epsilon}{12}$, we have 
     \begin{align}
      &\bbP\Big( (\langle \ba, \bx \rangle > \epsilon_0, \langle \ba, \bs \rangle > \epsilon_0) \cup (\langle \ba, \bx \rangle < -\epsilon_0, \langle \ba, \bs \rangle < -\epsilon_0)\Big) \ge 1- \frac{2\epsilon}{3}. 
     \end{align}
\end{lemma}
\begin{proof}
    We have  
     \begin{align}
      & \bbP\Big( (\langle \ba, \bx \rangle > \epsilon_0, \langle \ba, \bs \rangle > \epsilon_0) \cup (\langle \ba, \bx \rangle < -\epsilon_0, \langle \ba, \bs \rangle < -\epsilon_0)\Big) \nonumber \\
      & = \bbP( \langle \ba, \bx \rangle > \epsilon_0, \langle \ba, \bs \rangle > \epsilon_0) + \bbP(\langle \ba, \bx \rangle < -\epsilon_0, \langle \ba, \bs \rangle < -\epsilon_0) \\
      & \ge \bbP(\langle \ba, \bx \rangle > 0, \langle \ba, \bs \rangle > 0) - \bbP( |\langle \ba, \bx \rangle| \le \epsilon_0)  +  \bbP(\langle \ba, \bx \rangle < 0, \langle \ba, \bs \rangle < 0) - \bbP( |\langle \ba, \bs \rangle| \le \epsilon_0).
     \end{align}
    Note that by successively applying Lemmas~\ref{lem:first} and~\ref{lem:simple}, we have $\bbP(\langle \ba, \bx \rangle > 0, \langle \ba, \bs \rangle > 0) + \bbP(\langle \ba, \bx \rangle < 0, \langle \ba, \bs \rangle < 0) = 1-\rmd_\rmS(\bx,\bs) \ge 1-\frac{\epsilon}{2}$. In addition, because that $\langle \ba,\bx \rangle \sim \calN(0,1)$, we have 
    \begin{equation}
     \bbP(|\langle \ba,\bx \rangle| \le \epsilon_0) \le \epsilon_0 \sqrt{\frac{2}{\pi}},
    \end{equation}
    which is seen by trivially upper bounding the standard Gaussian density by $\frac{1}{\sqrt{2\pi}}$.
    Substituting $\epsilon_0 = \frac{\epsilon}{12}$, we obtain the desired inequality.
\end{proof}

Using Lemma~\ref{lem:noiseless_sep2} and following similar ideas to those in the proof of Theorem~\ref{thm:main_noiseless}, we deduce Corollary~\ref{coro:main_noiseless}. To avoid repetition, we provide only an outline below.

     We again construct a chain of nets and select $\bx_0,\bx_1,\ldots,\bx_l$ and $\bs_0,\bs_1,\ldots,\bs_l$ in the nets such that~\eqref{eq:bxbxbsbt} is satisfied. Let $\delta = c_1\epsilon^2$ with $c_1>0$ being a sufficiently small constant. From the triangle inequality, we obtain 
    \begin{equation}
        \|\bx_0-\bs_0\|_2 \le \frac{3\epsilon}{2}.
    \end{equation}
    Then, let 
    \begin{align}
          &\tilde{J}(\bx',\bs') := \Big\{i \in [m] \,:\, \Big(\langle \underline{\ba}_i,\bx' \rangle > \frac{\epsilon}{8}, \langle \underline{\ba}_i,\bs' \rangle >\frac{\epsilon}{8}\Big) 
            \cup \Big(\langle \underline{\ba}_i,\bx' \rangle < -\frac{\epsilon}{8}, \langle \underline{\ba}_i,\bs' \rangle < -\frac{\epsilon}{8}\Big) \Big\}.
    \end{align}
    Similar to~\eqref{eq:num_lb}, we can show that when $m = \Omega\left(\frac{k}{\epsilon}\log \frac{L r}{\delta}\right)$, with probability at least $1-e^{-\Omega(\epsilon m)}$, for all $(\bx',\bs')$ pairs in $G(M) \times G(M)$ with $\|\bx'-\bs'\|_2 \le \frac{3\epsilon}{2}$, we have 
    \begin{equation}\label{eq:num_lb2}
          |\tilde{J}(\bx',\bs')| \ge \left(1-\frac{3\epsilon}{2}\right) m.
    \end{equation}
    Combining~\eqref{eq:num_lb2} with~\eqref{eq:I2_def} and a suitable analog of \eqref{eq:bds_I1I2}, we obtain the desired result.
    
\section{Proof of Theorem \ref{thm:lb_noiseless2} (Noiseless Lower Bound)} 

The proof proceeds in several steps, given in the following subsections.

\subsection{Choice of Generative Model}

Recall that Theorem \ref{thm:lb_noiseless2}  only states the existence of some generative model for which $m = \Omega\left(k \log(Lr) + \frac{k}{\epsilon}\right)$ measurements are necessary.  Here we formally introduce the generative model, building on the approach from \cite{liu2020information} of generating group-sparse signals. 
We say that a signal in $\bbR^n$ is {\em $k$-group-sparse} if, when divided into $k$ blocks of size $\frac{n}{k}$,\footnote{To simplify the notation, we assume that $n$ is an integer multiple of $k$.  For general values of $n$, the same analysis goes through by letting the final $n - k\lfloor\frac{n}{k}\rfloor$ entries of $\bx$ always equal zero.} each block contains at most one non-zero entry.\footnote{More general notions of group sparsity exist, but for compactness we simply refer to this specific notion as $k$-group-sparse.} 

We construct an auxiliary generative model $\tilde{G} \,:\, B_2^k(r) \rightarrow \bbR^n$, and then normalize it to obtain the final model $G \,:\, B_2^k(r) \rightarrow \calS^{n-1}$.
Noting that $B_\infty^k\big(\frac{r}{\sqrt k}\big) \subseteq B_2^k(r) \subseteq B_\infty^k(r)$, we fix $x_{\mathrm{c}}, x_{\max} > 0$ and construct $\tilde{G}$ as follows: 
\begin{itemize}
    \item The output vector $\bx \in \bbR^n$ is divided into $k$ blocks of length $\frac{n}{k}$, denoted by $\bx^{(1)},\dotsc,\bx^{(k)} \in \bbR^{\frac{n}{k}}$.
    \item A given block $\bx^{(i)}$ is only a function of the corresponding input $z_i$, for $i=1,\dotsc,k$.
    \item The mapping from $z_i$ to $\bx^{(i)}$, $i \in [k-1]$ is as shown in Figure~\ref{fig:toy_gen}.  The interval $\big[-\frac{r}{\sqrt{k}},\frac{r}{\sqrt{k}}\big]$ is divided into $\frac{n}{k}$ intervals of length $\frac{2r\sqrt{k}}{n}$, and the $j$-th entry of $\bx^{(i)}$ can only be non-zero if $z_i$ takes a value in the $j$-th interval.  Within that interval, the mapping takes a ``double-triangular'' shape with extremal values $-x_{\max}$ and $x_{\max}$. 
    \item To handle the values of $z_i$ (with $i \in [k-1]$) outside $\big[-\frac{r}{\sqrt{k}},\frac{r}{\sqrt{k}}\big]$, we extend the functions in Figure~\ref{fig:toy_gen} to take values on the whole real line: For all values outside the indicated interval, each function value simply remains zero.
    \item Different from \cite{liu2020information}, we let the map corresponding to $z_k$ always produce $x_{n - n/k + 1} = x_{n - n/k + 2} = \ldots = x_{n-1} =0$ and $x_n = x_{\mathrm{c}} >0$, where the subscript `c' is used to signify ``constant''.  We allow $x_{\mathrm{c}} > x_{\max}$, as $x_{\max}$ only bounds the first $k-1$ non-zero entries.
    \item The preceding dot point leads to a Lipschitz-continuous function defined on all of $\bbR^k$, and we simply take $\tilde{G}$ to be that function restricted to $B_2^k(r)$.
\end{itemize}
\begin{figure}
    \begin{center}
        \includegraphics[width=0.5\textwidth]{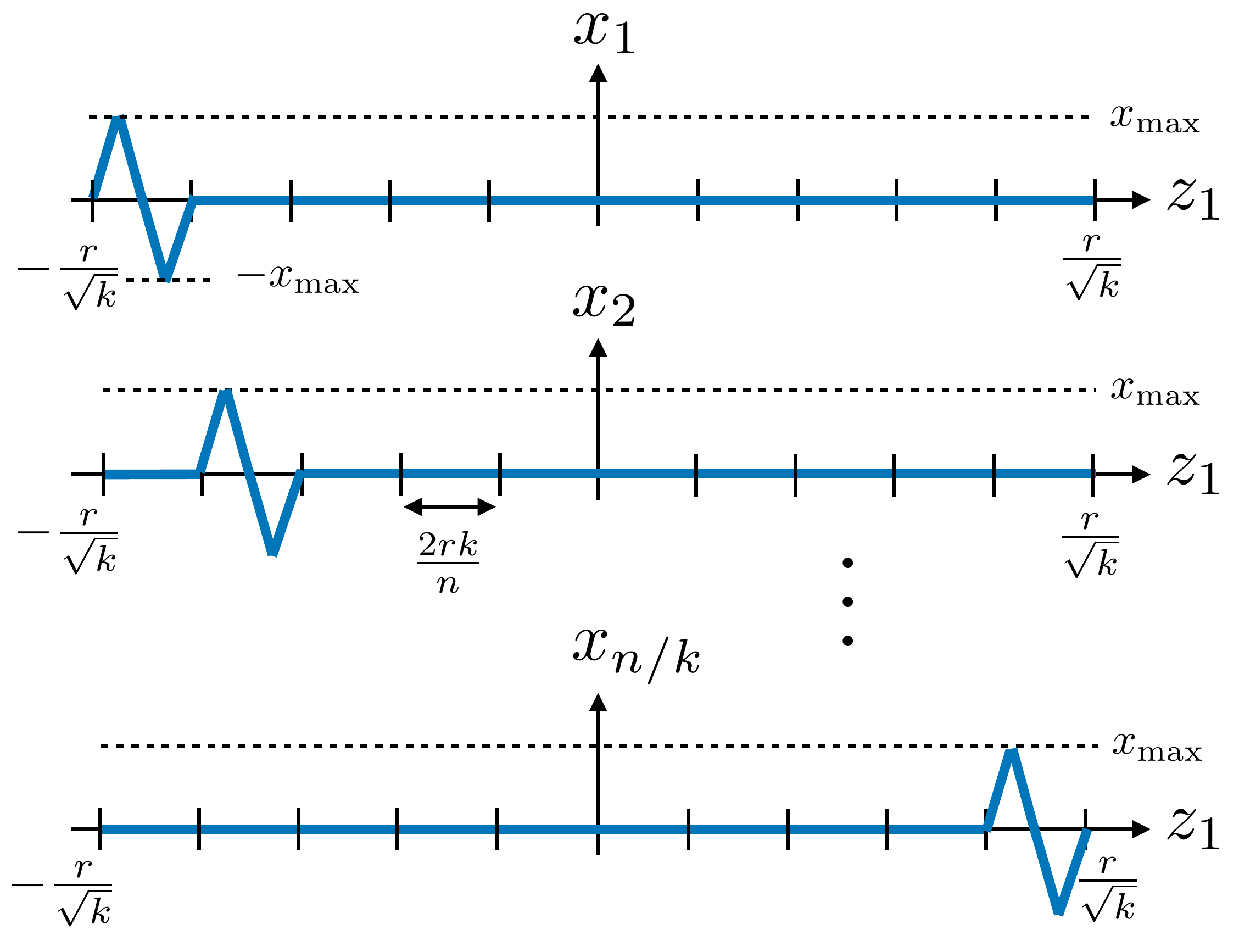}
    \end{center}
    \caption{Generative model that produces sparse signals.  This figure shows the mapping from $z_1 \to (x_1,\dotsc,x_{\frac{n}{k}})$, and the same relation holds for $z_2 \to (x_{\frac{n}{k}+1},\dotsc,x_{\frac{2n}{k}})$ and so on up to $z_{k-1} \to (x_{n-k+1-n/k},\dotsc,x_{n-n/k})$.} \label{fig:toy_gen} 
\end{figure}

To attain the final generative model used to prove Theorem \ref{thm:lb_noiseless2}, we take the output of $\tilde{G}$ and normalize it to be a unit vector: $G(\bz) = \frac{\tilde{G}(\bz)}{\|\tilde{G}(\bz)\|_2}$. We define
\begin{equation}
    \calX_k := \big\{ \bx \in \calS^{n-1} \,:\, \bx \text{ is } k \text{-group-sparse} \big\}. \label{eq:Xk}
\end{equation}
We observe the range $G(B_2^k(r))$ of $G$ is\footnote{For the extreme case that $x_{\mathrm{c}} =0$, it is easy to see that $G(B_2^k(r)) = \calX_k$ (ignoring the zero vector generated by $\tilde{G}$). It follows that for general $x_{\mathrm{c}}>0$, the range of the generative model is as given in~\eqref{eq:range_G}.  Indeed, $x_{\mathrm{c}}$ gets divided by $\|\tilde{G}(\bz)\|_2$, which in turn can take any value between $x_{\mathrm{c}}$ and $\sqrt{(k-1)x_{\max}^2 + x_{\mathrm{c}}^2}$. } 
\begin{align}\label{eq:range_G}
    \tilde{\calX}_k := \left\{\bx \in \calX_k \,:\, x_n \ge \frac{x_{\mathrm{c}}}{\sqrt{(k-1)x_{\max}^2 + x_{\mathrm{c}}^2}}\right\}. 
\end{align}
Furthermore, we have the following lemma regarding the Lipschitz continuity of $G$.
\begin{lemma}\label{lem:lip_G}
     The generative model $G \,:\, B_2^k(r) \rightarrow \calS^{n-1}$ defined above, with parameters $n$, $k$, $r$, $x_{\mathrm{c}}$, and $x_{\max}$, has a Lipschitz constant given by 
     \begin{equation}
        L =\frac{2n x_{\max}}{\sqrt{k}r x_{\mathrm{c}}}.
     \end{equation}
\end{lemma}
\begin{proof}
     From \cite[Lemma 1]{liu2020information}, we know that $\tilde{G}$ is $\tilde{L}$-Lipschitz with $\tilde{L} = \frac{2n x_{\max}}{\sqrt{k}r}$. It is straightforward to show that for any $\bx,\bx' \ne \mathbf{0}$, $\big\|\frac{\bx}{\|\bx\|_2} - \frac{\bx'}{\|\bx'\|_2}\big\| \le  \max \big\{\frac{1}{\|\bx\|_2}, \frac{1}{\|\bx'\|_2}\big\} \|\bx -\bx'\|_2$. Due to the choice of $x_n$ in our construction, we have $\|\tilde{G}(\bz)\|_2 \ge x_{\mathrm{c}}$ for any $\bz \in B_2^k(r)$; hence, for any $\bz_1,\bz_2 \in B_2^k(r)$, we have
     \begin{align}
          \|G(\bz_1)-G(\bz_2)\|_2  &= \left\|\frac{\tilde{G}(\bz_1)}{\|\tilde{G}(\bz_1)\|_2} - \frac{\tilde{G}(\bz_2)}{\|\tilde{G}(\bz_2)\|_2}\right\|_2 \\
          & \le \frac{1}{x_{\mathrm{c}}} \|\tilde{G}(\bz_1) - \tilde{G}(\bz_2)\|_2 \\
          & \le \frac{\tilde{L}}{x_{\mathrm{c}}} \|\bz_1 - \bz_2\|_2,
     \end{align}
    meaning that $G$ is $L$-Lipschitz with $L = \frac{2n x_{\max}}{\sqrt{k}r x_{\mathrm{c}}}$.
\end{proof}

\subsection{Proof of $\Omega\big(\frac{k}{\epsilon}\big)$ Lower Bound} \label{sec:lb1}

With the generative model in place that produces group-sparse signals, we proceed by following ideas from the 1-bit sparse recovery literature \cite{jacques2013robust,acharya2017improved}.  The following lemma is a simple modification of a lower bound for the packing number of the unit sphere.  The proof is deferred to Section \ref{app:pf_packing}.

\begin{lemma}\label{lem:lb_packingNum}
     For $\lambda \in (0,1)$, define
     \begin{equation}
        Z_k(\lambda) := \{\bz \in \calS^{k-1}\,:\, z_k \ge \lambda\}. \label{eq:Zk}
     \end{equation}
     Then, for any $k$ and $\epsilon \in \big(0,\frac{1}{2}\big)$, there exists a subset $\calC \subseteq Z_k(\frac{1}{2})$ of size $|\calC| \ge \left(\frac{c}{\epsilon}\right)^k$ (with $c$ being an absolute constant) such that for all $\bz,\bz' \in \calC$, it holds that $\|\bz-\bz'\|_2 > 2\epsilon$.  
\end{lemma}

The following lemma allows us to bound the number of distinct $\bb$ vectors (observed vectors) that can be produced by sparse signals.

\begin{lemma}{\em \hspace{1sp}\cite[Lemma.~8]{jacques2013robust}}\label{lem:num_orthant}
     For $m \ge 2k$, the number of orthants intersected by a single $k$-dimensional subspace in an $m$-dimensional space is upper bounded by $2^k \binom{m}{k}$. 
\end{lemma}

With the above lemmas in place, we proceed by deriving a lower bound on the minimal worst-case reconstruction error, defined as follows (and implicitly depending on a fixed but arbitrary measurement matrix $\bA$):
\begin{equation}
    \epsilon_{\mathrm{opt}} := \inf_{\psi(\cdot)} \sup_{\bx \in G(B_2^k(r))}  \|\bx - \psi(\bx)\|_2,
\end{equation}
where $\psi(\cdot)$ is the overall mapping from $\bx$ to its estimate $\hat{\bx}$, and is therefore implicitly constrained to depend only on $(\bA,\Phi(\bx))$ with $\Phi(\bx) = {\rm sign}(\bA\bx)$.  Note that our definition of $\epsilon_{\mathrm{opt}}$ differs from that in~\cite{jacques2013robust}, since we adopt a refined strategy more similar to \cite{acharya2017improved} to arrive at $\epsilon_{\mathrm{opt}} = \Omega\big(\frac{k}{m}\big)$ instead of the weaker $\epsilon_{\mathrm{opt}} = \Omega\big(\frac{k}{m+k^{3/2}}\big)$.
 
\begin{lemma}\label{lem:lb_noiseless}
     For the generative model $G$ described above with $x_{\mathrm{c}}$ and $x_{\max}$ chosen to satisfy $(k-1)x_{\max}^2 = 3x_{\mathrm{c}}^2$, we have
    \begin{equation}
        \epsilon_{\mathrm{opt}} = \Omega\left(\frac{k}{m}\right). 
     \end{equation}
\end{lemma}
\begin{proof}
   Note that $G(B_2^k(r))$ corresponds to a union of $N_{\mathrm{supp}} = \left(\frac{n}{k}\right)^{k-1}$ subsets $\cup_{i \in [N_{\mathrm{supp}}]} S_i$, with
    \begin{align}
        S_i := \bigg\{\bx \in \calX_k \,:\, \mathrm{supp}(\bx) \subseteq T_i, x_n \ge \frac{x_{\mathrm{c}}}{\sqrt{(k-1)x_{\max}^2 + x_{\mathrm{c}}^2}}\bigg\}, \label{eq:Si}
    \end{align}
    where the sets $T_i \subseteq [n]$ equal the $N_{\mathrm{supp}}$ possible supports of size $k$ for group sparse vectors. Substituting the assumption $(k-1)x_{\max}^2 = 3x_{\mathrm{c}}^2$ gives $ \frac{x_{\mathrm{c}}}{\sqrt{(k-1)x_{\max}^2 + x_{\mathrm{c}}^2}} = \frac{1}{2}$, and it follows that for any $i^* \in [N_{\mathrm{supp}}]$, we have
   \begin{align}
        \epsilon_{\mathrm{opt}} & = \inf_{\psi(\cdot)} \sup_{\bx \in G(B_2^k(r))}  \|\bx - \psi(\bx)\|_2  \\
       & \ge \inf_{\psi(\cdot)} \sup_{\bx \in S_{i^*}(\frac{1}{2})}  \|\bx - \psi(\bx)\|_2,
   \end{align}
   where we write \eqref{eq:Si} as $S_{i^*}(\frac{1}{2}):= \{\bx \in \calS^{n-1} \cap \calX_k \,:\, \mathrm{supp}(\bx) \subseteq T_{i^*}, x_n \ge \frac{1}{2} \}$ to highlight the fact that $\frac{x_{\mathrm{c}}}{\sqrt{(k-1)x_{\max}^2 + x_{\mathrm{c}}^2}} = \frac{1}{2}$.  Hence, it suffices to derive the lower bound for $\epsilon_{\mathrm{opt}}^* := \inf_{\psi(\cdot)} \sup_{\bx \in S_{i^*}(\frac{1}{2})}  \|\bx - \psi(\bx)\|_2$. 

    To simplify notation, we assume in the following that the preceding infimum over $\psi(\cdot)$ is attained by some $\psi^*(\cdot)$.\footnote{If not, a similar argument applies with $\psi^*_\zeta(\cdot)$ satisfying $\sup_{\bx \in S_{i^*}(\frac{1}{2})}  \|\bx - \psi^*(\bx)\|_2 \le \epsilon_{\mathrm{opt}}^* + \zeta$ for an arbitrarily small $\zeta$.}  By Lemma~\ref{lem:lb_packingNum}, there exists a set $\calC \subseteq S_{i^*}(\frac{1}{2})$, and a constant $c>0$ such that $|\calC| \ge \big(\frac{c}{\epsilon_{\mathrm{opt}}^*}\big)^k$, and for all $\bx,\bs \in \calC$, $\|\bx-\bs\|_2 > 2\epsilon_{\mathrm{opt}}^*$. In addition, from Lemma~\ref{lem:num_orthant}, the cardinality of the set $\widehat{\calX}^* := \{\hat{\bx} \in \bbR^n \,:\, \hat{\bx} = \psi^*(\bx) \text{ for some } \bx \in S_{i^*}(\frac{1}{2})\}$ satisfies $|\widehat{\calX}^*| \le 2^k \binom{m}{k}$, since each distinct outcome $\bb \in \{-1,1\}^m$ produces at most one additional estimated vector.

    For any $\bx \ne \bs \in \calC$, we must have $\psi^*(\bx) \ne \psi^*(\bs)$. To see this, suppose by contradiction that there exist $\bx \ne \bs \in \calC$ such that $\psi^*(\bx) = \psi^*(\bs)$. Because $\|(\bx-\psi^*(\bx))-(\bs-\psi^*(\bs))\|_2 = \|\bx-\bs\|_2 > 2\epsilon_{\mathrm{opt}}^*$, we have that at least one of $\|\bx-\psi^*(\bx)\|_2$ and $\|\bs-\psi^*(\bs)\|_2$ is larger than $\epsilon_{\mathrm{opt}}^*$, which contradicts the condition that $\sup_{\bx \in S_{i^*}(\frac{1}{2})}  \|\bx - \psi^*(\bx)\|_2 \le \epsilon_{\mathrm{opt}}^*$.

    Hence, combining the above cardinality bounds, we find
    \begin{equation}
        2^k \binom{m}{k} \ge |\widehat{\calX}^*| \ge |\calC| \ge  \left(\frac{c}{\epsilon_{\mathrm{opt}}^*}\right)^k,
    \end{equation}
   and applying the inequality $\binom{m}{k} \le \left(\frac{em}{k}\right)^k$, it follows that $\epsilon_{\mathrm{opt}}^* \ge \frac{ck}{2em}$ as desired.
\end{proof}

Lemma~\ref{lem:lb_noiseless} implies that for any $\epsilon \in \big(0,\frac{1}{2}\big)$, to ensure that there exists a reconstruction function $\psi(\cdot)$ such that $\sup_{\bx \in G(B_2^k(r))} \|\bx - \psi(\bx)\|_2 \le \epsilon$, we require that the number of samples $m$ satisfies $m = \Omega\left(\frac{k}{\epsilon}\right)$. 

\subsection{Proof of $\Omega\left(k \log(Lr)\right)$ Lower Bound} \label{sec:lb2}

The proof of the  $m = \Omega\left(k \log(Lr)\right)$ lower bound follows a similar high-level approach to that of $m = \Omega\left(\frac{k}{\epsilon}\right)$.  We first state the lower bound in terms of $n$ as follows.

\begin{lemma}
     For any $\epsilon \le \frac{\sqrt{3}}{4\sqrt{2}}$ and any reconstruction function $\phi(\cdot)$, in order to attain the recovery guarantee $\sup_{\bx \in G(B_2^k(r))} \|\bx - \phi(\bx)\|_2 \le \epsilon$, the number of samples $m$ must satisfy $m = \Omega\left(k \log \frac{n}{k}\right)$.
\end{lemma}
\begin{proof}
    Recall from \eqref{eq:Xk} that $\calX_k$ contains the $k$-group sparse signals on the unit sphere.  For any $\lambda \in (0,1)$, let 
    \begin{equation}
        S(\lambda) := \{\bx \in \calX_k \,:\, x_n \ge \lambda\}. \label{eq:calU}
     \end{equation}
    We claim that for some constant $c > 0$ and any $\epsilon \le \frac{\sqrt{3}}{4\sqrt{2}}$, there exists a subset $\calC \subseteq S(\frac{1}{2})$ such that $\log |\calC| \ge c k \log\left(\frac{n}{k}\right)$, and for all $\bx,\bs \in \calC$, it holds that $\|\bx-\bs\|_2 > 2\epsilon$. 
    To see this, consider the set
    \begin{align}
         \calU := \bigg\{\bx \in \calX_k \,:\, x_n = \frac{1}{2}, \, x_i \in \bigg\{0,\sqrt{\frac{3}{4(k-1)}} \bigg\} ~\forall i \le n-1, \|\bx\|_0 = k \bigg\}
    \end{align}
    of group-sparse signals with exactly $k$ non-zero entries, $k-1$ of which take the value $\sqrt{\frac{3}{4(k-1)}}$.   By a simple counting argument, we have $|\calU| = \left(\frac{n}{k}\right)^{k-1}$. 

    Let $k' = k-1$ for convenience, and for each $\bx \in \calU$, let $\bv \in \big\{ 1,\dotsc, \frac{n}{k}\big\}^{k'}$ be a length-$k'$ vector indicating which index in each block of the group-sparse signal (except the $k$-th one) is non-zero.  Then, for $\bx,\bx' \in \calU$ and the corresponding $\bv,\bv'$, we have
    \begin{equation}
        \|\bx - \bx'\|_2^2 = \frac{3}{4k'} \rmd_{\rm H}'(\bv,\bv'), \label{eq:l2_to_l0}
    \end{equation}
    where $\rmd_{\rm H}'(\bv,\bv') = \sum_{i=1}^n \bone\{ v_i \ne v'_i \}$ is the unnormalized Hamming distance.  By the Gilbert-Varshamov bound, we know that there exists a set $\calV$ of signals in $\big\{ 1,\dotsc, \frac{n}{k}\big\}^{k'}$ whose pairwise unnormalized Hamming distance is at least $d$, and with the number of elements satisfying
    \begin{align}
        |\calV| &\ge \frac{ (\frac{n}{k'})^{k'} }{ \sum_{j=0}^{d-1}( n/k - 1)^j } \\
            &\ge \frac{ (\frac{n}{k'})^{k'} }{ d ( \frac{n}{k'} )^d }.
    \end{align}
    Setting $d = \frac{k'}{2}$, we find that $\log |\calV| = \Omega\big( k \log \frac{n}{k} \big)$, and by \eqref{eq:l2_to_l0}, we have that the corresponding $\bx$ sequences are pairwise separated by at least a squared distance of $\frac{3}{8}$.  This gives us the desired set $\calC$ stated following \eqref{eq:calU}.

    By the triangle inequality, every $\bx \in \calC$ must have a different outcome $\Phi(\bx)$, since if two have the same outcome then their $2\epsilon$-separation (along with the triangle inequality) implies that the decoder's output cannot be $\epsilon$-close to both.  Since $m$ binary measurements can result in $2^m$ possible outcomes, it follows that $2^m \ge |\calC|$, and hence $m \ge \log_2 |\calC| = \Omega\left(k \log \frac{n}{k}\right)$.
\end{proof}  

Combining the preceding two lower bounds, we readily deduce Theorem \ref{thm:lb_noiseless2}: From Lemma~\ref{lem:lip_G}, the generative model $G$ that we used above has a Lipschitz constant given by
 \begin{equation}
    L = \frac{2n x_{\max}}{\sqrt{k}r x_{\mathrm{c}}} = \frac{n}{k} \frac{2 \sqrt{k} x_{\max}}{r x_{\mathrm{c}}},
 \end{equation}
which implies that when $(k-1)x_{\max}^2 = 3x_{\mathrm{c}}^2$, the condition $m = \Omega\left(k \log \frac{n}{k} \right)$ is equivalent to $m = \Omega\left(k \log(Lr)\right)$. Combining with the lower bound $\Omega\left(\frac{k}{\epsilon}\right)$ derived in Section \ref{sec:lb1}, we complete the proof of Theorem \ref{thm:lb_noiseless2}.

\subsection{Proof of Lemma~\ref{lem:lb_packingNum} (Lower Bound on the Packing Number)} \label{app:pf_packing}

    We first recall the following well-known lower bound on the packing number of the unit sphere.
    
    \begin{lemma}{\em \hspace{1sp}\cite[Ch.~13]{lorentz1996constructive}}
        \label{lem:wellKnownPacking}
         For any $k$ and $\epsilon \in \big(0,\frac{1}{2}\big)$, there exists a subset $\calC \subseteq \calS^{k-1}$ of size $|\calC| \ge \left(\frac{c}{\epsilon}\right)^k$ (with $c$ being an absolute constant) such that for all $\bz,\bz' \in \calC$, it holds that $\|\bz-\bz'\|_2 > 2\epsilon$.
    \end{lemma}

    Recall that Lemma~\ref{lem:lb_packingNum} is stated for $\lambda = \frac{1}{2}$.
    Fix $\tilde{\lambda} \in [\frac{1}{2},\frac{3}{4}]$, and consider the set $\calT(\tilde{\lambda}):= \{\bz \in \calS^{k-1} \,:\, z_k = \tilde{\lambda}\}$. Applying Lemma~\ref{lem:wellKnownPacking} to $\sqrt{1-\tilde{\lambda}^2}\calS^{k-2}$, we obtain that for any $\epsilon>0$, there exists a subset $\calC'(\tilde{\lambda}) \subseteq \calT(\tilde{\lambda})$, and a constant $c'(\tilde{\lambda}) >0$, such that $|\calC'(\tilde{\lambda})| \ge \big(\frac{c'(\tilde{\lambda})}{\epsilon}\big)^{k-1}$, and for all $\bz,\bz' \in \calC'(\tilde{\lambda})$, it holds that $\|\bz-\bz'\|_2 > 2\epsilon$.  In addition, since we consider $\tilde{\lambda} \in [\frac{1}{2},\frac{3}{4}]$, we have $\min_{\tilde{\lambda} \in [\frac{1}{2},\frac{3}{4}]} c'(\tilde{\lambda}) > 0$.

    For the final entry, observe that there exists a set $\calL \subseteq [\frac{1}{2},\frac{3}{4}]$ with $|\calL| \ge \frac{1}{8\epsilon}$ such that for all $a,b \in \calL$, it holds that $|a-b| > 2 \epsilon$. Then, considering $\cup_{l \in [\calL]} \calT(l)$ and letting $\calC:= \cup_{l \in [\calL]} \calC'(l) \subseteq Z_k\big(\frac{1}{2}\big)$ (see \eqref{eq:Zk}). We deduce that there exists a constant $c>0$ such that $|\calC| \ge \big(\frac{c}{\epsilon}\big)^k$, and for all $\bx,\bs \in \calC$ it holds that $\|\bx-\bs\|_2 > 2\epsilon$.

\bibliographystyle{IEEEtran}
\bibliography{techReports,JS_References}

\end{document}